\documentclass[10pt]{article} % For LaTeX2e

\usepackage[preprint]{rlj}          
\usepackage[ruled,vlined]{algorithm2e}
\usepackage{amssymb}            % Defines common symbols like \mathbb R
\usepackage{mathtools}          % Extends amsmath, providing common math tools
\usepackage{mathrsfs}           % Enables \mathscr, which can work in cases that \mathcal does not
\usepackage{graphicx}           % For including images
\usepackage{subcaption}         % Allows for the use of subfigures and subcaptions
\usepackage[space]{grffile}     % For spaces in image names
\usepackage{url}                % For displaying URLs
\usepackage{lipsum}             % For placeholder text

\usepackage{graphicx}
\usepackage{amsmath,amsfonts}
\usepackage{bm}% Required for inserting images
\usepackage{amssymb,mathtools,amsthm}
\usepackage[dvipsnames]{xcolor}
\usepackage[nameinlink,capitalize]{cleveref} % For more automatic, context-aware cross-references

\let\emptyset\varnothing

\def\equationautorefname#1#2\null{Eq.#1(#2\null)}

\newtheorem{definition}{Definition}
\newtheorem{theorem}{Theorem}
\newtheorem{lemma}{Lemma}
\newtheorem{proposition}{Proposition}
\newtheorem{corollary}{Corollary}
\newtheorem{example}{Example}

\hypersetup{
  linkcolor  = BrickRed,
  citecolor  = MidnightBlue,
  urlcolor   = Aquamarine,
  colorlinks = true,
}

\usepackage{tikz} % the tikz package  
\usetikzlibrary{shapes.geometric,arrows,arrows.meta} % to implement the corresponding shapes and arrows, these are used.  
\usetikzlibrary{shapes.multipart}
\usetikzlibrary{calc}
\usetikzlibrary{positioning}
\tikzset{%
  >={Latex[width=2mm,length=2mm]},
  base/.style = {rectangle, rounded corners, 
                 minimum width=1.6cm, minimum height=0.8cm, 
                 text centered},
  start/.style = {base, draw=black, fill=orange!30},
  stop/.style = {base, draw=black, fill=red!30},
  process/.style = {base, draw=black, fill=orange!15,},
  criteria/.style = {ellipse, text centered, fill=Aquamarine!40, inner sep=0.05cm},
}

\title{From monoliths to modules: Decomposing transducers for efficient world modelling}

\setrunningtitle{Composing and decomposing world models}

\author{Alexander Boyd\textsuperscript{1,2}, Franz Nowak\textsuperscript{3,4}, David Hyland\textsuperscript{5}, Manuel Baltieri\textsuperscript{6,1}, \\Fernando E. Rosas\textsuperscript{7-9}}

\emails{\small{alecboy@gmail.com, franz.nowak@inf.ethz.ch, manuel\_baltieri@araya.org,
david.hyland@stx.ox.ac.uk,
f.rosas@sussex.ac.uk}}

\affiliations{
$^{1}$\textbf{Department of Informatics, University of Sussex}\\
$^{2}$\textbf{Beyond Institute for Theoretical Science (BITS)}\\
$^{3}$\textbf{ETH Zürich}\\
$^{4}$\textbf{Principles of Intelligent Behaviour in Biological and Social Systems (PIBBSS)}\\
$^{5}$\textbf{Department of Computer Science, University of Oxford}\\
$^{6}$\textbf{Araya Inc.}\\
$^{7}$\textbf{Sussex AI and Sussex Centre for Consciousness Science, University of Sussex}\\
$^{8}$\textbf{Centre for Complexity Science and Center for Psychedelic Research, Department of Brain Sciences, Imperial College London}\\
$^{9}$\textbf{Center for Eudaimonia and Human Flourishing, University of Oxford}

}

\begin{document}

\maketitle  

\begin{abstract}
World models have been recently proposed as sandbox environments in which AI agents can be trained and evaluated before deployment. 
While realistic world models often have high computational demands, this can often be alleviated by exploiting the fact that real-world scenarios tend to involve subcomponents that interact in a modular manner. 
In this paper, we explore this idea by 
developing a framework for decomposing complex world models represented by transducers, a class of models generalising POMDPs. 
Whereas the composition of transducers is well understood, our results clarify how to invert this process deriving sub-transducers operating on distinct input-output subspaces, 
enabling parallelizable and interpretable alternatives to monolithic world modelling that can support distributed inference. 
Overall, these results lay groundwork 
for bridging the computational efficiency required for real-world inference and the structural transparency demanded by AI safety.
\end{abstract}

\section{Introduction}

Advances in deep reinforcement learning have produced agents that can operate in increasingly complex environments, from dexterous robotic control and large-scale games to open-ended dialogue and tool use~\citep{andrychowicz2020learning,berner2019dota,ouyang2022training,rajeswaran2017learning}. As these systems become more powerful and autonomous and are considered for deployment in high-stakes settings, questions about reliability, safety, and interpretability are beginning to take central stage~\citep{bengio2024international,glanois2024survey,tang2024prioritizing}. One way to address these questions is by adopting world models as controlled testbeds, serving as synthetic environments in which agents can be trained, probed, and stress-tested before deployment~\citep{dalrymple2024towards,bruce2024genie,diaz2023connecting}. However, the usefulness of such sandboxing depends critically on our ability to build world models that are faithful enough to guarantee that the agent’s behaviour in simulation can be transferred to real-world settings~\citep{rosasai,ding2025understanding}.

A world model~\citep{ding2025understanding} can be formally described as a transducer that turns input sequences (e.g., actions) into output sequences (e.g., observations and rewards). 
This transduction, however, happens via interwoven dynamics often involving co-occurring physical processes alongside social interactions, biological rhythms, or engineered feedback loops. A surfer, for instance, must simultaneously navigate the ocean's turbulence and the coordinated movements of their companions; naively modelling all of this as a single high-dimensional mechanism makes learning and inference expensive and blurs the underlying structure. 
For world models to be useful, they should instead expose \emph{modularity}: distinct components should track different aspects of the world, while still interacting and combining to produce coherent global behaviour. 
Prior work related to the modularity of transduction has mostly approached this from a bottom-up perspective, focusing on how small components can be composed into layered networks that capture complex input-output relationships~\citep{hopcroft1979introduction,mohri-1997-finite,mohri2002weighted}. 
Similarly, approaches such as factored Markov decision problems (MDPs)~\citep{boutilier2000stochastic,guestrin2003efficient} and reward machines~\citep{icarte2018using,dohmen2022inferring} have proposed bottom-up procedures to design complex environments or reward structures by composing simple elements in various topologies. 

While composition explains how modular systems can be built, for AI interpretability and safety we often face the inverse problem: given a large, seemingly entangled world model, how can we decompose it into simpler, interpretable pieces without sacrificing predictive power?

This question has been explored from the perspective of model-based reinforcement learning, where agents learn a model of the underlying transition dynamics of the environment that can be used for several downstream tasks such as planning and training policies~\citep{hafner2019dream,hafner2025trainingagentsinsidescalable,hansen2023td}. %However, simulating and learning the transition dynamics of high-dimensional data (such as video data) in ``entangled'' latent spaces is complex, computationally expensive, and less interpretable. 
%Leveraging the fact that ordinary human-scale environments often consist of sparse interactions between multiple distinct entities, s
Several works have proposed techniques for incorporating inductive biases to make neural networks represent the world as a set of discrete objects, which have demonstrated benefits related to generalisation~\citep{goyal2021recurrent,zhao2022toward,feng2023learning}, sample efficiency~\citep{rodriguezsanchez2025pixelsfactorslearningindependently}, interpretability~\citep{mosbach2025sold}, compositional generation~\citep{baek2025dreamweaver}, and robustness to noise~\citep{liu2023learning}.

\begin{figure*}[h!]
\centering
\includegraphics[width=\columnwidth]{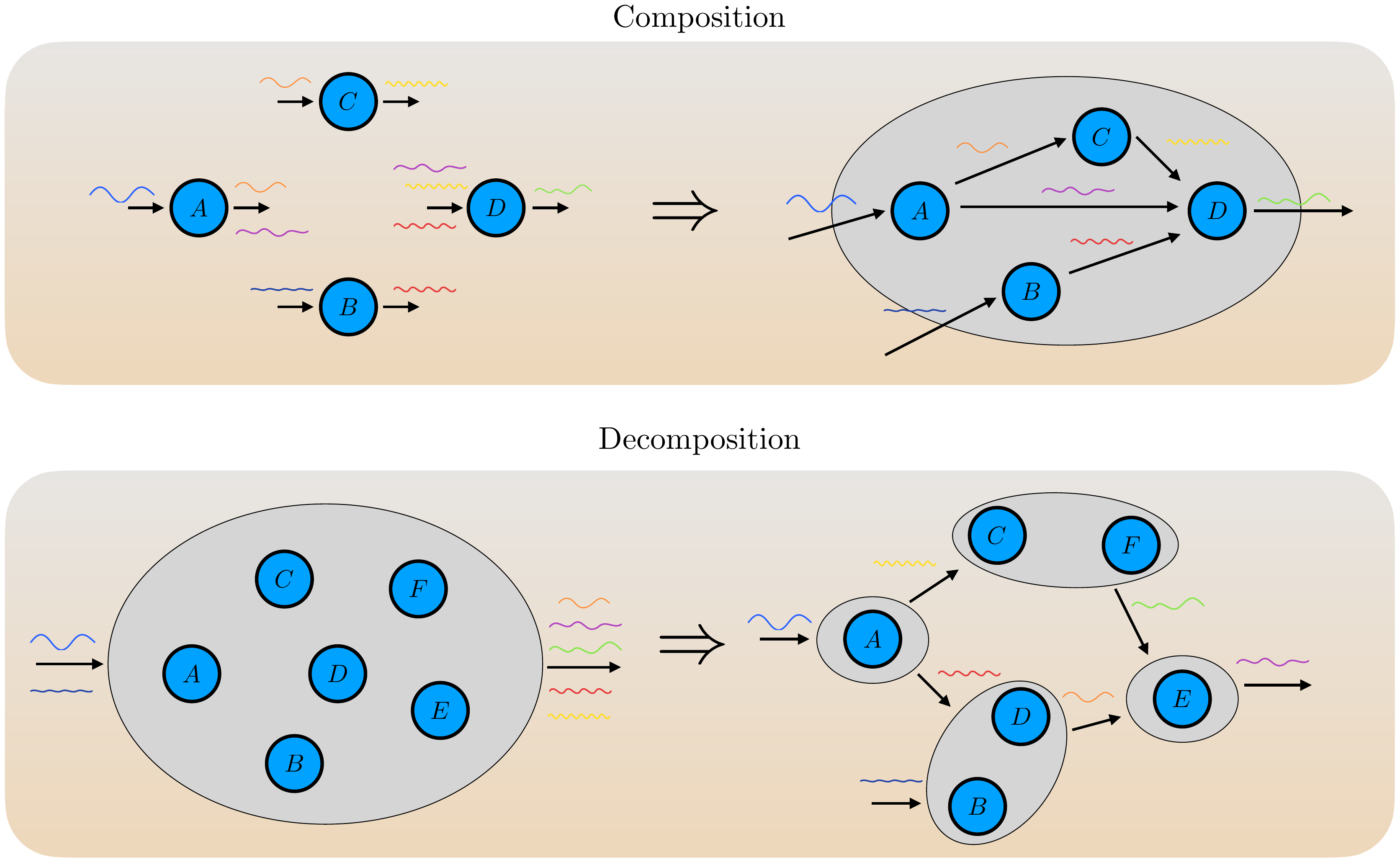}
\caption{\centering\small \emph{In this work, we first present a method for composing stochastic environments into larger ones. We
then use this framework to identify procedures that reverse the process, decomposing complex environments
into simpler, modular subcomponents.}}
\label{fig:Overview} 
\end{figure*}

Inspired by these results, here we develop a general framework to address this question in the language of computational mechanics~\citep{crutchfield1989inferring,shalizi2001computational}, which is used to understand when and how a complicated transducer can be factored into simpler interacting sub-transducers. %This provides a way to move from a single high-dimensional transducer to a collection of modular low-dimensional ones, each associated with a functional subsystem of the world model. 
Crucially, this factorisation does more than aid understanding: it enables distributed inference, where each sub-transducer can update its own beliefs locally, given only the parts of the input-output history that matter to it.  Thus, as shown in \autoref{fig:Overview}, our analysis reveals how transducers can be both composed and decomposed, moving between levels of complexity. 

In summary, in this paper we take the perspective of a designer who wishes to model an agent--environment interface in a way that is faithful enough for prediction, but also structured enough to support analysis and control. Within this perspective, transducers provide a canonical representation, and our main contribution is to show how such representations can themselves be decomposed into interpretable, modular components. The rest of the paper formalises transducers and their interfaces (\autoref{sec:transducers}), introduces a general notion of transducer composition (\autoref{sec:composition_trasducers}), to then develop  decomposition algorithms (\autoref{sec:decomp_framework}) and coarse-graining approaches (\autoref{sec:coarse-graining}). 
We also investigate how this approach can be used to decompose the causal states of an AI system doing inference tasks (\autoref{sec:decomposition_compmech}), opening the way to extend ongoing efforts to interpret the internal representations of transformers~\citep{shai2025transformers,piotrowski2025constrained,shai2026transformers} to high-dimensional settings. 
Related work is discussed in \autoref{sec:related_work_trasnducer_comp} and \autoref{sec:causal_discovery}.

\begin{tcolorbox}[title={\textbf{Overview and roadmap}}]
This work develops an information-theoretic framework for
composing and decomposing modular world models. 
The roadmap ahead and main contributions are as follows:
\begin{enumerate}
  \item \textbf{Interfaces and transducers as a common language (\autoref{sec:transducers}).}
  We use the notions of \emph{interfaces} and \emph{transducers} as a general
  representation for world models that can capture agents, environments,
  and predictive models within a single framework.

  \item \textbf{Composition of world models as networks of transducers (\autoref{sec:composition_trasducers}).}
  We show how transducers compose in series, parallel, and convergent
  configurations, and build a unified view that treats a network
  of transducers as a single composite world model.

  \item \textbf{Decomposition with latent variables via Intransducibility (\autoref{sec:factoring_latents}).}
  We introduce \emph{Intransducibility}, an information-theoretic diagnostic that detects when a set of variables cannot be generated by
  a causal transducer, and use it to factor a monolithic transducer into
  a sparse network when latent variables are available.

  \item \textbf{Decomposition without latents via acausality (\autoref{sec:factoring_nolatents}).}
  We define an \emph{acausality} measure that can be computed directly
  from observable signals, and show how it recovers minimally acausal
  wiring diagrams in the absence of explicit latent variables.

  \item \textbf{Coarse-graining networks of transducers (\autoref{sec:coarse-graining}).}
  We characterize when blocks of a transducer network can be merged
  without changing its observable interface, yielding principled
  multiscale coarse-grainings of world models.

  \item \textbf{Factoring causal states and belief states (\autoref{sec:decomposition_compmech}).}
  We relate our framework to $\epsilon$-transducers~\citep{barnett2015computational} and predictive state representations~\citep{littman2001predictive}, showing how minimal predictive representations inherit and
  expose the modular structure of the underlying transducer network.
\end{enumerate}
\end{tcolorbox}

\section{Preliminaries}
\label{sec:transducers}

\subsection{Interfaces}

Agents and environments can be described functionally by how they map input sequences to output sequences.  In probabilistic terms, a mapping of this kind is fully specified by a collection of conditional distributions of outputs given inputs, which we call an \emph{interface}.
\begin{definition}
An \textbf{interface} $\mathcal{I}$ is a conditional input-output process characterized by the family of consistent conditional probabilities 
\begin{align}
    \mathcal{I}(Y|X) \equiv \big\{\Pr(Y_{0:t}|X_{0:t}), \forall t\in\{0,1,2,\dots\}\big\},
\end{align}

where $X_t$ are inputs and $Y_t$ are outputs.
\end{definition}
Here, uppercase variables denote random variables and lowercase represent specific realizations.  We use the shorthand $X_{0:t} \equiv (X_0, \ldots, X_{t-1})$ for ordered sequences of random variables, and the random variable notation $\Pr(A|B) \equiv \{ \Pr(A=a|B=b) \}_{a \in \mathcal{A},b \in \mathcal{B}}$ is used to describe all elements of a conditional probability.  An interface specifies a semi-infinite stochastic process for every semi-infinite input string, describing a transformation from increasingly long strings of inputs $x_{0:t}$ to corresponding lengths of outputs $y_{0:t}$ --- like an information ratchet that performs stochastic operations on time series in sequence~\citep{boyd2016identifying, mandal2012work}.  However, this temporal interpretation is only appropriate for certain subclasses that we call \textbf{\textit{causal interfaces}}, for which future inputs $\overrightarrow{X}_t \equiv X_{t:\infty}$ cannot causally influence past outputs $\overleftarrow{Y}_t \equiv Y_{0:t}$.  Hence, causal interfaces can be interpreted as a circuit, as shown in \autoref{fig:Condensed_Interface}.

\begin{figure*}
\centering
\includegraphics[width=\columnwidth]{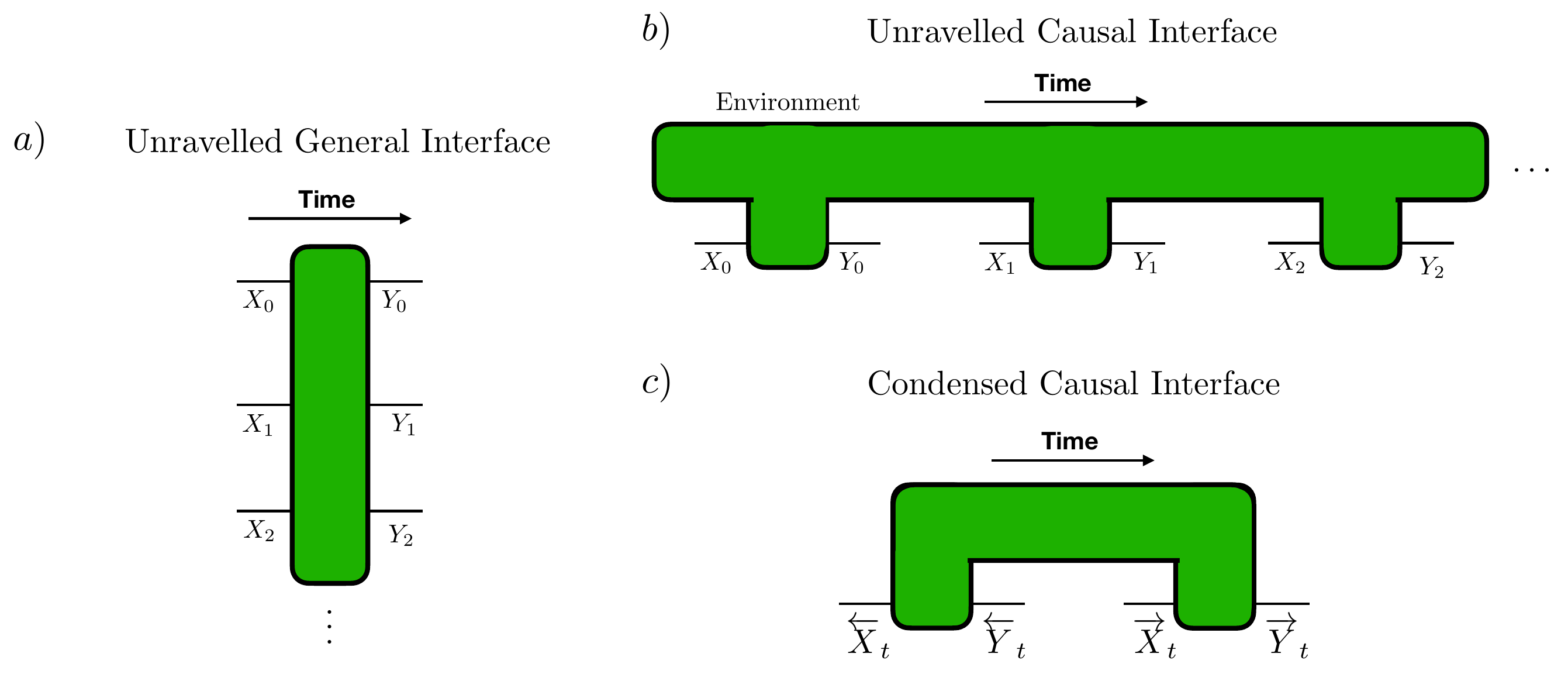}
\caption{\centering\small \emph{A general interface and two illustrations of a causal interface:} An unravelled general interface (a) takes a semi-infinite sequence of inputs $X_0,X_1,X_2 \cdots$ and stochastically transforms it to a semi-infinite sequence of outputs $Y_0,Y_1,Y_2 \cdots$, without any constraints on dependencies between inputs and outputs.  An unravelled causal interface (b) shows individual inputs $X_t$ and outputs $Y_t$ unravelled into a semi-infinite sequence in time.  The same object can be condensed (c) into a mapping from input pasts $\overleftarrow{X}_t$ and futures $\overrightarrow{X}_t$ to output pasts $\overleftarrow{Y}_t$ and futures $\overrightarrow{Y}_t$. }
\label{fig:Condensed_Interface} 
\end{figure*}

While causal interfaces are well-suited to describe actions and agents in perception-action loops and feedback\footnote{Interfaces can also be used to describe agents and feedback, as discussed in \autoref{app:Interfaces}}, this manuscript exclusively addresses \emph{feedforward} networks of interfaces, where past outputs cannot influence future inputs. 
In this context, the condition for causality can be expressed in terms of a probabilistic relation
\begin{align}    \Pr(\overleftarrow{Y}_t|\overleftarrow{X}_t,\overrightarrow{X}_t)=\Pr(\overleftarrow{Y}_t|\overleftarrow{X}_t).
\label{eq:nonanticipatory}
\end{align}
An interface $\mathcal{I}[Y|X]$ in a feedforward context is causal iff \autoref{eq:nonanticipatory} holds, and implies that the interface is anticipation-free or \emph{nonanticipatory}~\citep{barnett2015computational, rosasai}. If the nonanticipatory condition is not satisfied in a feedforward network, then the interface somehow utilizes future inputs before receiving them, violating causality.  Thus, causal interfaces are the only ones that are physically realizable in time.

\begin{figure*}
\centering
\includegraphics[width=.6\columnwidth]{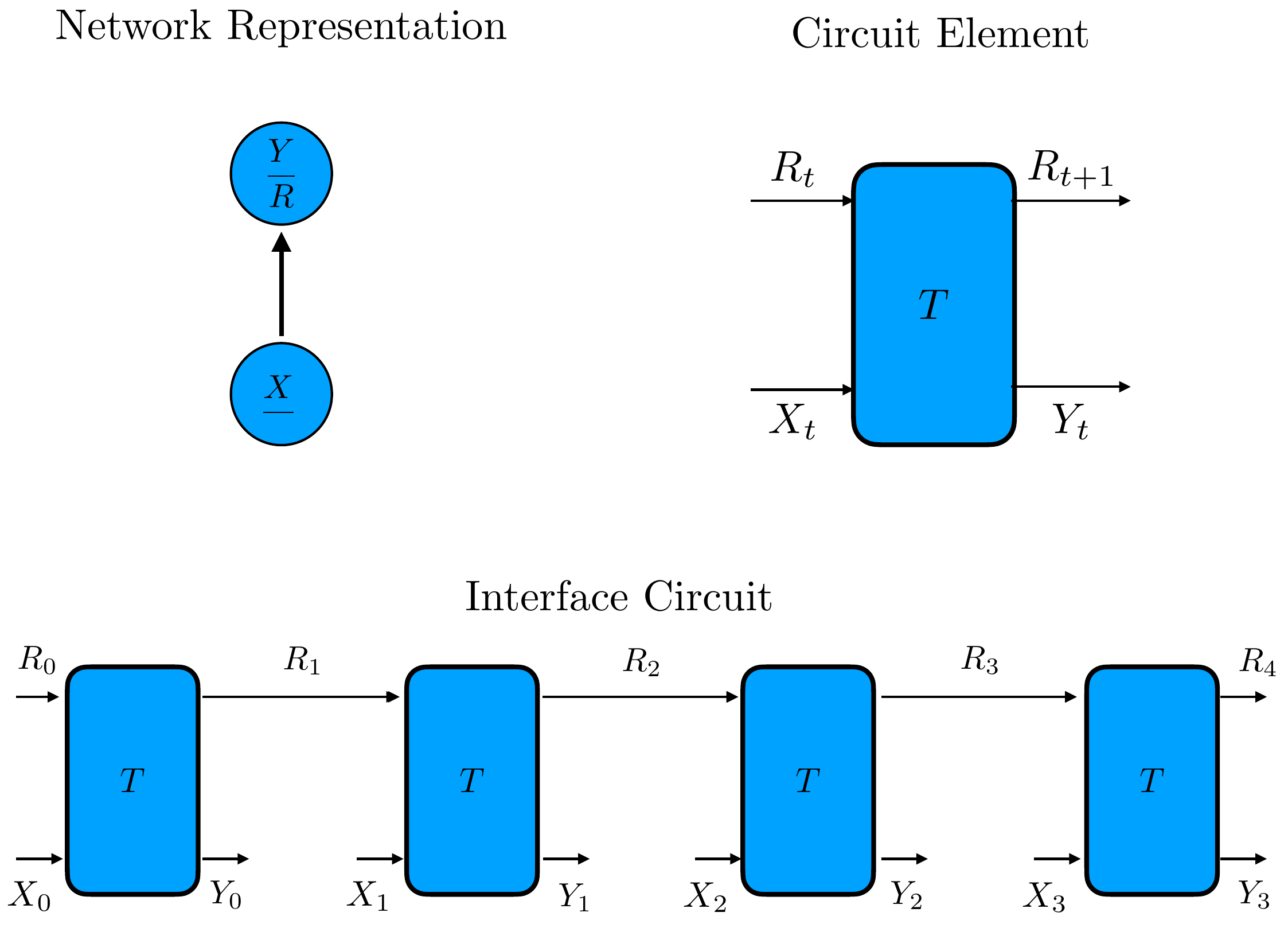}
\caption{\centering\small A transducer is a general model that transforms an input process $X$ to an output process $Y$ using a latent process $R$ as memory, which can be used to generate interfaces. The \emph{network representation} draws an arrow from the input process $X$ with unknown source ($\frac{X}{\null}$) to the output process $Y$ with latent variable $R$ ($\frac{Y}{R}$).  The \emph{circuit element} representation takes two inputs ($X_t$ and $R_t$) to two outputs ($Y_t$ and $R_{t+1}$).  The \emph{interface circuit} representation exhibits multiple timepoints resulting from applying the circuit element in series.}
\label{fig:Transducer_Circuit} 
\end{figure*}

\subsection{Transducers}

There are several models in the literature that can be used to implement causal interfaces, including Markov Decision Processes (MDP)~\citep{sutton1998introduction} and partially observed Markov Decision Processes (POMDP)~\citep{kaelbling1998planning}. These approaches are generalized by \emph{transducers}, which specify a repeating circuit relying on a latent variable that evolves in tandem with inputs and outputs to carry memory from past to future~(\autoref{fig:Transducer_Circuit}).

\begin{definition}\label{def:transducer}
    A \textbf{transducer} $T$ is a 
    tuple $\{\mathcal{X},\mathcal{Y},\mathcal{R},\{ T^{(y|x)}_{r \rightarrow r'} \}_{r,r' \in \mathcal{R},x \in \mathcal{X}, y \in \mathcal{Y}} \}$ consisting of an input alphabet $\mathcal{X}$, an output alphabet $\mathcal{Y}$, a latent memory alphabet $\mathcal{R}$, 
    and a stochastic kernel $T^{(y|x)}_{r \rightarrow r'}$ that determines the following probability:
    \begin{align}
    \label{eq:kernel}
        T^{(y|x)}_{r \rightarrow r'}= \Pr(Y_t=y, R_{t+1}=r' | X_t=x, R_t=r).
    \end{align} 
\end{definition}

Transducers can be used to implement interfaces by applying the kernel repeatedly and summing over all latent memory sequences as follows:
\begin{align}
\label{eq:Transducer_Process}
\Pr(Y_{0:t}=y_{0:t}|X_{0:t}=x_{0:t})=\sum_{r_{0:t+1}\in \mathcal{R}^{t+1}}\Pr(R_0=r_0)\prod_{i=0}^{t-1}T^{(y_i|x_i)}_{r_i \rightarrow r_{i+1}},
\end{align}
where $\Pr(R_0)$ is a prior distribution for $R_0$. 
We say that $T$ is a \emph{transducer presentation} of the interface $\mathcal{I}$ if there exists an initial distribution $\Pr(R_0)$ such that the transducer satisfies \autoref{eq:Transducer_Process} for all $\Pr(Y_{0:t}=y_{0:t}|X_{0:t}=x_{0:t})\in\mathcal{I}(Y|X)$.  

In line with past explorations of transducers~\citep{barnett2015computational,boyd2024thermodynamic,boyd2018thermodynamics}, here we define the stochastic kernel $T^{(y|x)}_{r \rightarrow r'}$ to be independent of the time-step $t$.  As shown in \autoref{fig:Transducer_Circuit}, this means the same operation $T$ is applied repeatedly, making it \textit{\textbf{mechanistically stationary}} (see~\autoref{app:MechanicallyStationary}). Transducers closely align with Hidden Markov Models, which implement stochastic processes~\citep{jurgens2021shannon} without considering inputs. The generality of transducers is confirmed by our first result, which identifies nonanticipatory interfaces with collections of conditional distributions that can be generated by transducers (see proof in~\autoref{app:Nonanticipatory Transducers}).

\begin{theorem}\label{thm:transducer_interface_equivalence}
A collection of conditional distributions of the form $\Pr(Y_{0:t}=y_{0:t}|X_{0:t}=x_{0:t})$ constitutes a causal interface iff it has a transducer presentation.
\end{theorem}

%This result emphasizes the generality of transducers as a model for environments and agents.  As long as these entities respect causality (information flowing from past to future), a transducer can fully capture their behaviour.

\section{Networks of composed transducers}
\label{sec:composition_trasducers}

As an input-to-output map, transducers can be used as elementary components of more complex circuits.  Here, we investigate feedforward circuits, leaving recursive ones for future work.

\subsection{Composing two transducers}
\label{sec:transducersComposition}

Two transducers of the form $T= \{\mathcal{X},\mathcal{Y},\mathcal{R}, T^{(y|x)}_{r \rightarrow r'}\}$ 
and 
$U=\{\mathcal{X \times Y},\mathcal{Z},\mathcal{S},U^{(z|xy)}_{s \rightarrow s'} \}$ 
can be composed as shown in \autoref{fig:Composition_Circuit}, where both the input and output of $T$ are used to drive $U$ and generate the interface $\mathcal{I}[YZ|X]$. We formalize this in the next definition.

\begin{definition}
\label{def:Composition}
    The composition of transducers $T= \{\mathcal{X},\mathcal{Y},\mathcal{R}, T^{(y|x)}_{r \rightarrow r'}\}$ 
    and 
    $U=\{\mathcal{X \times Y},\mathcal{Z},\mathcal{S},U^{(z|xy)}_{s \rightarrow s'} \}$ is a new transducer $V = \{\mathcal{X},\mathcal{Y}\times\mathcal{Z},\mathcal{R}\times\mathcal{S}, V^{(yz|x)}_{rs \rightarrow r's'}\}$ with input alphabet $\mathcal{X}$, output alphabet $\mathcal{Y}\times\mathcal{Z}$, latent memory alphabet $\mathcal{R}\times\mathcal{S}$, and a stochastic kernel $V^{(yz|x)}_{rs \rightarrow r's'} \equiv  U^{(z|xy)}_{s \rightarrow s'} T^{(y|x)}_{r \rightarrow r'}$ that determines the following probability:
    \begin{align}
    V^{(yz|x)}_{rs \rightarrow r's'} = \Pr(Z_t=z,Y_t=y,R_{t+1}=r',S_{t+1}=s'|X_t=x,R_t=r,S_t=s).
\end{align}
\end{definition}

\begin{figure*}
\centering
\includegraphics[width=0.8\columnwidth]{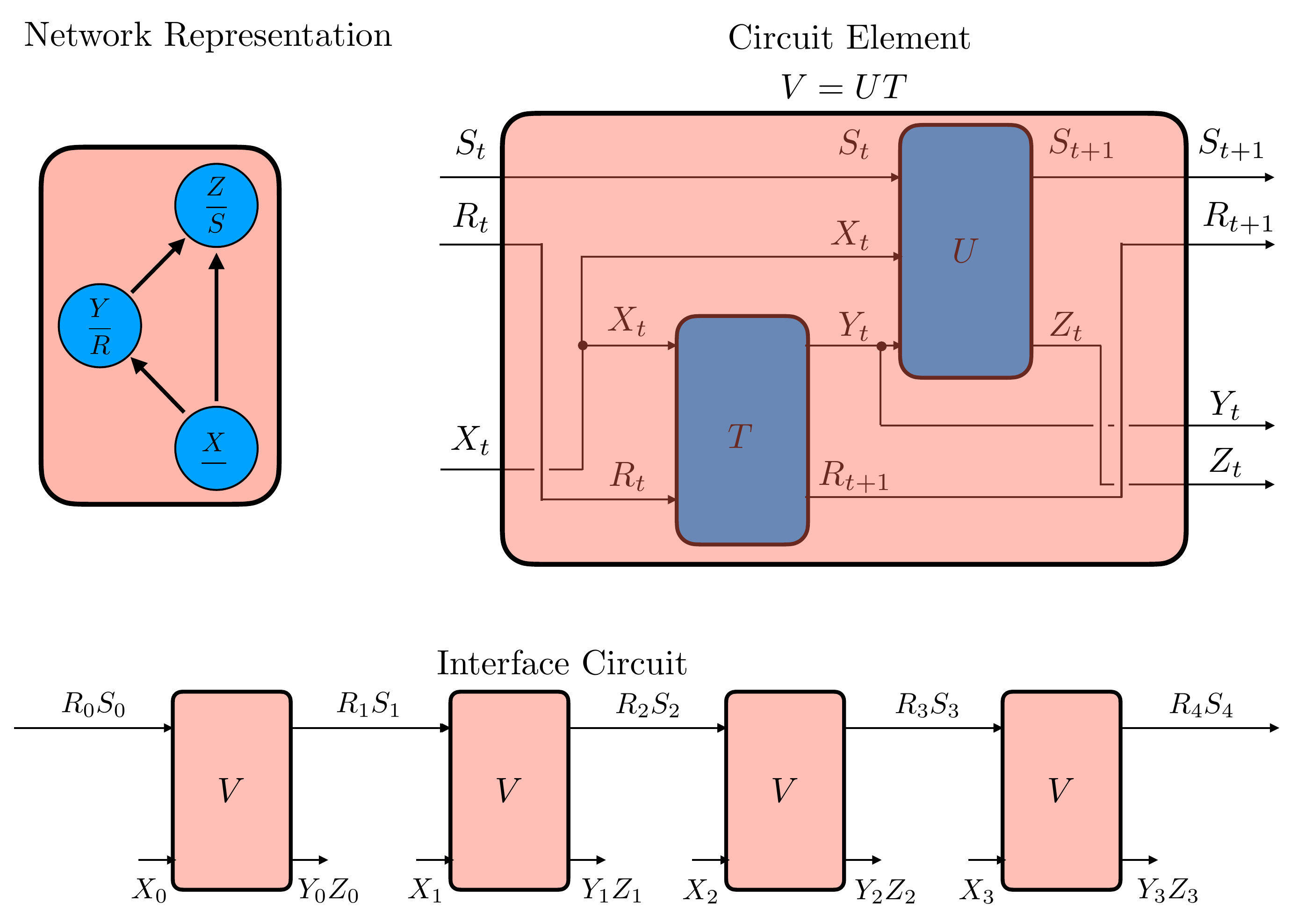}
\caption{\centering\small The network in the top left shows the most general way of composing two transducers, $T$ with latent states $R$, inputs $X$, and outputs $Y$, and U with latent states $S$, inputs $XY$, and outputs $Z$.  
A circuit element that implements this composite transducer is shown on the top right, with time proceeding from left to right. 
The composite transducer $V$, when applied in sequence at the bottom, produces the interface from $X$ to $YZ$.}
\label{fig:Composition_Circuit} 
\end{figure*}

For mathematical convenience, we introduce linear operators for the kernel.  For transducer $T$, the kernel has a linear operator on a vector space of the latent memory $\mathcal{R}$:
\begin{align}
    \hat{T}^{(y|x)} = \sum_{r,r'}T_{r \rightarrow r'}^{(y|x)} \mathbf{e}_{r'} \mathbf{e}_{r}^\intercal.
\end{align}
The notation $\mathbf{e}_{r} $ represents a column vector with a $1$ in the $r$th row and $0$s everywhere else, and $\mathbf{e}_{r}^\intercal $ is its transpose.  The linear operator notation allows us to evaluate the probability of an output sequence given a sequence of inputs via the product of operators: 
\begin{align}
\Pr(Y_{0:L}=y_{0:L}|X_{0:L}=x_{0:L})= \mathbf{1}^\intercal \hat{T}^{(y_{0:L}|x_{0:L})} \mathbf{p}_{R_0}.
\end{align}
Here, 
\begin{align}
    \hat{T}^{(y_{0:L}|x_{0:L})} \equiv\prod_{t=0}^{L-1}\hat{T}^{(y_t|x_t)},
\end{align}
is the composition of the transducer operators for the particular input-output sequence $y_{0:L}|x_{0:L}$, 
\begin{align} 
\mathbf{1}  \equiv \sum_{r} \mathbf{e}_r,
\end{align}
is the vector of all $1$s in the latent memory basis,
\begin{align}
\mathbf{p}_{R_0}   \equiv  \sum_{r} \Pr(R_0=r) \mathbf{e}_r
\end{align}
is the vector of initial memory probability.   
Using this notation, we can state our next result, whose proof is found in \autoref{app:kronecker}.

\begin{proposition}\label{res:kronecker}
    The operator associated to the composition of transducers is the Kronecker product of the linear operators of each transducer:
\begin{align}
\hat{V}^{(yz|x)}& =  \hat{T}^{(y|x)} \otimes \hat{U}^{(z|xy)} \nonumber
\\ & = \sum_{r,r',s,s'} T_{r \rightarrow r'}^{(y|x)} U_{s \rightarrow s'}^{(z|xy)} \mathbf{e}_{r'} \mathbf{e}_r^\intercal \otimes \mathbf{e}_{s'} \mathbf{e}_{s}^\intercal,
\end{align}
which operates on the joint vector space of the latents, which has dimension $|\mathcal{S}|\times|\mathcal{R}|$.
\end{proposition}

The composition shown in \autoref{def:Composition} inherits the properties of the Kronecker product, suggesting that this form of composition is fundamental.  This operation is \emph{associative but not commutative.}  Any other form of composition, such as the cascade product, series composition, or parallel composition discussed in \autoref{app:Other Composition Types}, can be built up using this fundamental operation and other tools discussed later in the manuscript.  This includes building larger networks through multiple compositions, marginalizing observed variables, and eliminating dependencies between variables, all of which we will cover.

\subsection{Special cases of transducer composition}
\label{sec:special-cases}

Using the composite transducer $V=UT$ as the baseline, we see that by restricting dependencies and making the transducers agnostic to certain inputs, we obtain various sub-classes of transducer composition that can be partially ordered by their restrictions (see~\autoref{fig:3_Observable_Dependencies}).  Some notable examples of these sub-classes include:
\begin{enumerate}[label=\alph*)]
    \item \emph{Series: }If the second transducer $U$ does not depend on the input $X$ of the transducer $T$, then $\hat{U}^{(z|xy)}=\hat{U}^{(z|y)}$ and the two transducers compose sequentially, with the composite transducer outputting both $U$'s and $T$'s outputs.
    \item \emph{Convergent: } If the first transducer $T$ is input-agnostic, i.e., $\hat{T}^{(y|x)}=\hat{T}^{(y)}$, then information from $X$ and $Y$ converges to produce $Z$.
    \item \emph{Divergent: } If the second transducer $U$ is agnostic to input $Y$, i.e., $\hat{U}^{(z|xy)}=\hat{U}^{(z|x)}$, then information from $X$ diverges to produce $Y$ and $Z$ in parallel.
\end{enumerate}

\begin{figure*}
\centering
\includegraphics[width=\columnwidth]{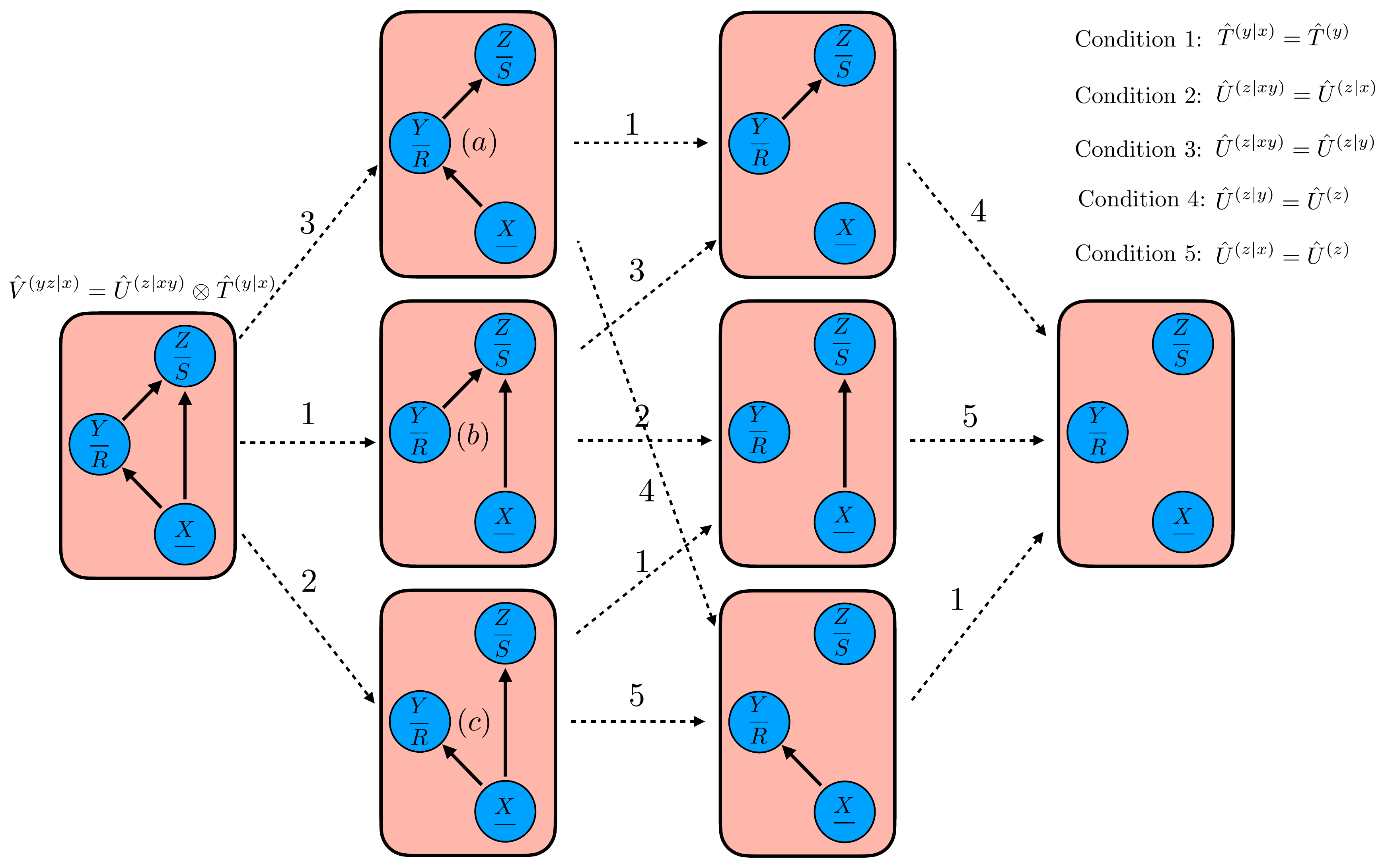}
\caption{
\centering
\small{Lattice of sub-classes of transducer composition, ordered according to the number of restrictions they consider. Pruning edges from left to right corresponds to limiting dependencies on the inputs --- these limitations are enumerated in conditions 1 through 5.  The arrows with dotted lines are labelled with the number of the condition that is necessary to prune each edge.  In this lattice, we highlight three notable cases: (a) series, (b) convergent, and (c) divergent composition.}}
\label{fig:3_Observable_Dependencies} 
\end{figure*}

To illustrate some of the concepts we have discussed so far, we consider a simple example in the familiar domain of gridworlds that considers transducers composed in series and in divergent manner.

\begin{example}
    \label{ex:gridworld}
    Consider an agent navigating a two-dimensional grid $\mathcal{G} = \{0, 1, \ldots, W-1\} \times \{0, 1, \ldots, H-1\}$, with action space $\mathcal{A} = \{u,d,l,r\}$ corresponding to movement in different directions. The state $s_t = (x_t, y_t)$ encodes the agent's position at time $t$.
    
    We can decompose the environment dynamics into three transducers (see~\autoref{fig:gridworld-transducers}):
    \begin{enumerate}
        \item \textbf{Horizontal position transducer} $T_x$: Tracks the $x$-coordinate, with latent state $R_x = \{0, \ldots, W-1\}$, updating via
        \[
            r_x^{(t+1)} = \begin{cases}
                \min(r_x^{(t)} + 1, W-1) & \text{if } a_t = r \\
                \max(r_x^{(t)} - 1, 0) & \text{if } a_t = l \\
                r_x^{(t)} & \text{otherwise}
            \end{cases}
        \]
        
        \item \textbf{Vertical position transducer} $T_y$: Tracks the $y$-coordinate analogously, with latent state $R_y = \{0, \ldots, H-1\}$.
        
        \item \textbf{Reward transducer} $T_r$: Receives position outputs $(x_t, y_t)$ from 
        $T_x$ and $T_y$, and emits a scalar reward $r_t \in \mathbb{R}$. The reward transducer can be used to represent non-Markovian reward functions by utilising a memory latent $M$, and is equivalent to a \emph{probabilistic reward machine}~\citep{dohmen2022inferring}.
    \end{enumerate}
     Notice that in some cases, the latent memory state of a transducer can also serve as its observable output (e.g., the positions of $T_x$ and $T_y$ are emitted directly as outputs $X$ and $Y$ respectively). The composition structure can be written as $T_{\text{env}} = T_r \circ (T_x \parallel T_y)$, where $\parallel$ denotes parallel composition and $\circ$ denotes sequential composition. 
\end{example}

\begin{figure}[t]
    \centering
    \includegraphics[width=\columnwidth]{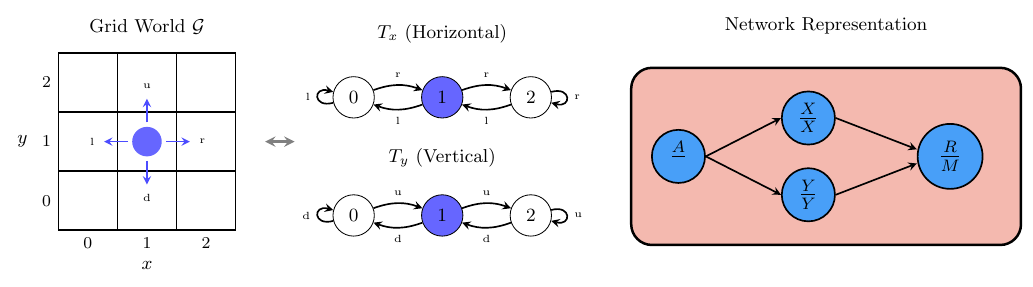}
    \caption{A simple $3 \times 3$ grid world environment decomposed into parallel position transducers feeding a sequential reward
    transducer. The horizontal ($T_x$) and vertical ($T_y$) position transducers operate independently based on the action stream, while the reward transducer $T_r$ combines their outputs to produce scalar rewards based on the position history of the agent.}
    \label{fig:gridworld-transducers}
    \end{figure}

Note that in the figures of one transducer or two composed transducers such as in Figure~\ref{fig:Composition_Circuit}, the outputs $Y$ and $Z$ were conditional on the input $X$, meaning that the composed transducers generate the conditional distribution $\Pr(YZ|X)$.  However, this is functionally equivalent to considering $X$ as the result of a memoryless IID input-agnostic transducer.

\subsection{Transducer networks}

Let us now use \autoref{def:Composition} to build  networks of transducers. For this, let us consider a collection of $N$ transducers $\{T(n)\}_{n \in \{0, \cdots ,N-1\}}$, whose respective latent memory alphabets are denoted by $\mathcal{R}(n)$ and their observable alphabets by $\mathcal{X}(n)$. 
Since outputs flow through the network, these variables may function as both inputs and outputs of the transducers, depending on the context. 
For simplicity, here we focus on networks without loops, leaving the treatment of feedback for future work.

For convention, we assume that the numbering of the transducers is such that it respects their causal ordering, meaning that $T_{n}$ can only influence $T_{n'}$ if $n' > n$, which guarantees \emph{no feedback} in the system. 
Thus, we assume that $\mathcal{X}(n)$ serves as both the output of transducer $n$ and part of the input of transducers $n+1$ onwards. 
Most generally, the linear operator of a transducer $T(n)$ depends on all previous observable outputs and can be expressed as $ \hat{T}(n)^{(x(n)|x(0:n))}$
where $x(0:n) \equiv \{x(0), \cdots, x(n-1)\}$. These assumptions lead to the next, which is a direct consequence of \autoref{res:kronecker}.

\begin{proposition}
    Under the above conditions, 
    the resulting joint transducer operator of the composition of $N$ transducers $\otimes_{n=0}^{N-1} T_n$ is given by
    \begin{align}
    \hat{T}^{(x(0:N))}(0:N)= \bigotimes_{n=0}^{N-1}\hat{T}(n)^{(x(n)|x(0:n))},
    \end{align}
    where $\bigotimes$ denotes the Kronecker product.
\end{proposition}

\begin{figure*}
\centering
\includegraphics[width=\columnwidth]{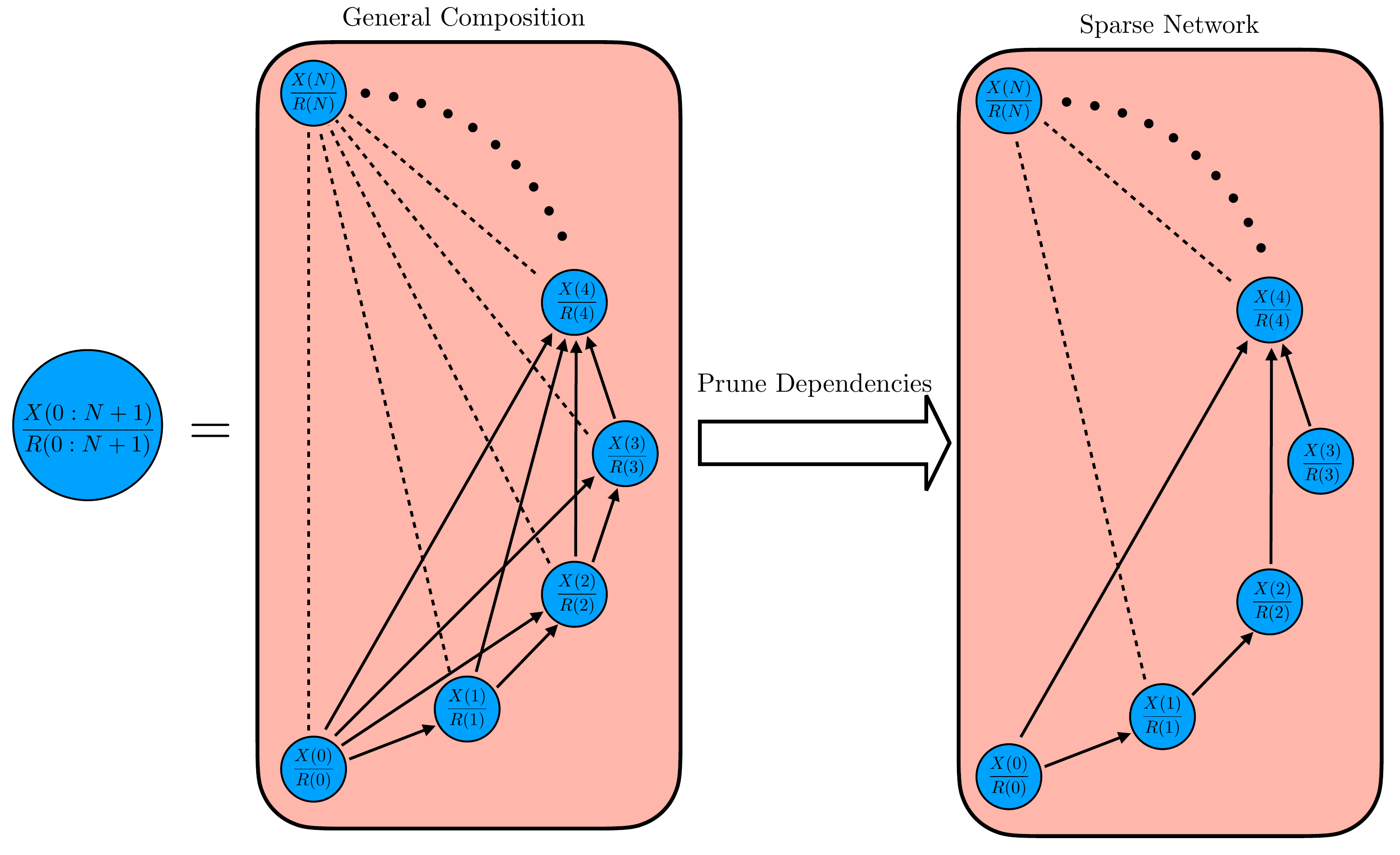}
\caption{
\centering
\small{A network of $N+1$ transducers can be expressed as a single transducer (left), or as a compositional network (middle).  In this fully connected composition, we see that the $n$th transducer, with random variables $\frac{X(n)}{R(n)}$ depends on all prior outputs $X(0:n)=\{X(0), \cdots, X(n-1)\}$.  This composition allows us to examine sparse networks (right), where dependencies are pruned between elements.}}
\label{fig:N_Observable_Dependencies} 
\end{figure*}

The condition of having no feedback in the dependency of transducers means that information proceeds from one layer of the network system to the next, and the interdependency can be expressed as a directed acyclic graph (DAG), as shown in \autoref{fig:N_Observable_Dependencies}. 
Graphically, the network representation of a composition of multiple transducers consists of a set of nodes $n$, each labelled $\frac{X(n)}{R(n)}$, where $X(n)$ is the random variable for the output process of node $n$, and $R(n)$ is the latent memory process.  If the transducer at node $n$ has a directed edge to the transducer at node $n'$, it implies that transducer $n'$ depends on $T(n)$'s output process $X(n)$. 
Thus, the arrows from node $n$ determine which latent variables and outputs are affected by the output process $X(n)$.  For instance, \autoref{fig:Composition_Circuit} shows arrows going from $X|$ to $Y|R$ and $Z|S$, meaning that the current outputs $Y_t$ and $Z_t$ and the next latent states $R_{t+1}$ and $S_{t+1}$ are dependent on the current output $X_t$.  This is shown explicitly in the circuit diagram next to the network representation of the composed transducer.

\begin{example}
    \label{ex:multi-objective-transducer}
     Consider settings with complex or multi-objective reward structures, where the reward is better understood as the aggregate of distinct interacting feedback loops. For instance, suppose that the agent's task is to collect resources that stochastically spawn at random grid locations and deliver them to a central depot, but has a finite carrying capacity (see~\autoref{fig:gridworld-transducers_extended}). The global reward function can be factored into two interacting transducers: a collection transducer $T_c$ and a delivery transducer $T_d$. The collection transducer takes as input the location of the agent as well as the output of an inventory transducer $T_I$. When the agent's position coincides with a resource, $T_c$ emits a small collection reward $r_c$ if the agent's inventory is not full. The delivery transducer takes the inventory state as an input alongside the agent's position and emits a larger delivery reward $r_d$ depending on how many resources were delivered. By factoring the reward structure in this way, an agent may be able to learn separate policies for efficient resource collection and ultimate delivery to the depot and compose them to solve the task.
\end{example}

As shown in \autoref{fig:Composition_Circuit}, it is also possible for nodes to leave latent variables, which is denoted as $\frac{X(n)}{\null}$.  This indicates that this node functions as an input to the transducer to which it has arrows.

If a subset of nodes $\mathcal{N}' \subseteq \{0,\cdots , N-1 \}$ has directed edges to the node $n$, that means that the corresponding linear operator has simplified dependencies as follows:
\begin{align}
    \hat{T}(n)^{(x(n)|x(0:n))}=\hat{T}(n)^{(x(n)|x(\mathcal{N}'))},
\end{align}
where $x(\mathcal{N}')\equiv \{x(n')|n' \in \mathcal{N}'\}$.  \autoref{fig:N_Observable_Dependencies} shows an example where the dependencies between nodes are restricted by choosing a subset of all possible antecedent processes.  This includes the possibility where there are no antecedent processes, such as for the node $\frac{X(3)}{R(3)}$.  This node only creates information and feeds it into the network in \autoref{fig:N_Observable_Dependencies}, but does not receive any information from prior nodes.  Such a node without antecedents can be conditioned out, and used as an input process $\frac{X(3)}{\null}$.

\begin{figure}[t]
    \centering
    \includegraphics[width=0.9\columnwidth]{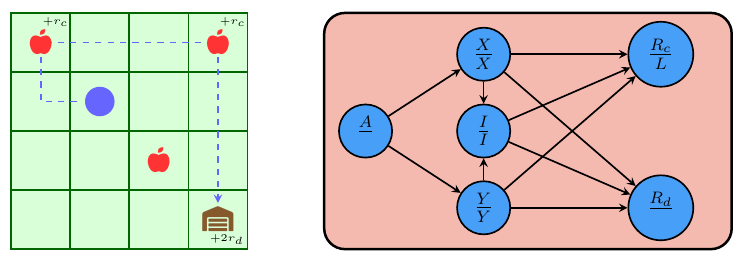}
    \caption{Extended gridworld example with multi-objective reward structure. 
    \textbf{Left:} A $4 \times 4$ grid environment in which an agent (blue circle) must collect resources (yielding reward $+r_c$) that spawn at random locations and deliver them to a central depot (yielding reward $+2r_d$). The dashed line shows a sample trajectory.
    \textbf{Right:} A compositional transducer network depicting the scenario described in Example \ref{ex:multi-objective-transducer}. The action transducer $T_A$ feeds horizontal ($T_X$) and vertical ($T_Y$) position transducers, which in turn feed an inventory transducer $T_I$ tracking the agent's carrying state. The collection reward transducer $T_{R_c}$ and delivery reward transducer $T_{R_d}$ each depends on the agent's position and inventory, factorizing the global reward into two separate components. In addition, the collection transducer tracks locations $L$ on the grid as a latent so that it knows when to output rewards for resource collection.}
    \label{fig:gridworld-transducers_extended}
    \end{figure}

In summary, using sequences of transducer composition one can build any feed-forward network of transducers. 
\emph{Any} set of feedforward dependencies between processes is accessible through this formalism.

\begin{example}
\label{ex:mars_rover}
As a running example, consider an agent controlling a \emph{solar-powered Mars rover}. At each time $t$, the agent inputs motor commands $X_t$, and the environment returns observations $Y_t = (I_t, V_t)$ consisting of a camera image $I_t$ to monitor the terrain in front of it and a battery voltage reading $V_t$. The voltage reading $V_t$ is dependent on the latent battery state $B_t$, which is influenced by the motor drain, and by an exogenous weather process that modulates the illuminance $L_t$ measured by a photometer as a result of solar radiation incident on the rover, which is not controlled by the agent.

A monolithic view would suggest a single, complex mapping $X \to Y$ that uses a single latent memory state $R_t$ to jointly encode the navigation, atmospheric, and battery states. From a learning or inference perspective, this can be inefficient, as the agent is required to predict a high-dimensional input stream using a representation that entangles multiple weakly-coupled mechanisms within a single latent space.

However, using the composition machinery described above, the system can be represented as a composition of three smaller transducers, each tracking a distinct subsystem but producing the same overall interface:
\begin{enumerate}
    \item \textbf{Navigation} $T_{\mathrm{nav}}$: Maps motor inputs $X$ to camera images $I$ via position latents $P$.
    \item \textbf{Weather} $T_{\mathrm{sol}}$: An input-agnostic transducer that measures illuminance $L$ via weather state latents $W$.
    \item \textbf{Power} $T_{\mathrm{pwr}}$: A convergent transducer that integrates motor drain $X$ and illuminance $L$ to output voltage $V$ via the latent battery state $B$.
\end{enumerate}
% Composing yields a joint operator for $(I,S,V)$ given $X$.
% If we ultimately want an interface from $X$ to $I$ only, we can coarse-grain by marginalising out $S$ and $V$ (cf.\ Section~5).

\begin{figure*}
    \centering
    \label{fig:Mars_Rover}
    \includegraphics[width=0.9\columnwidth]{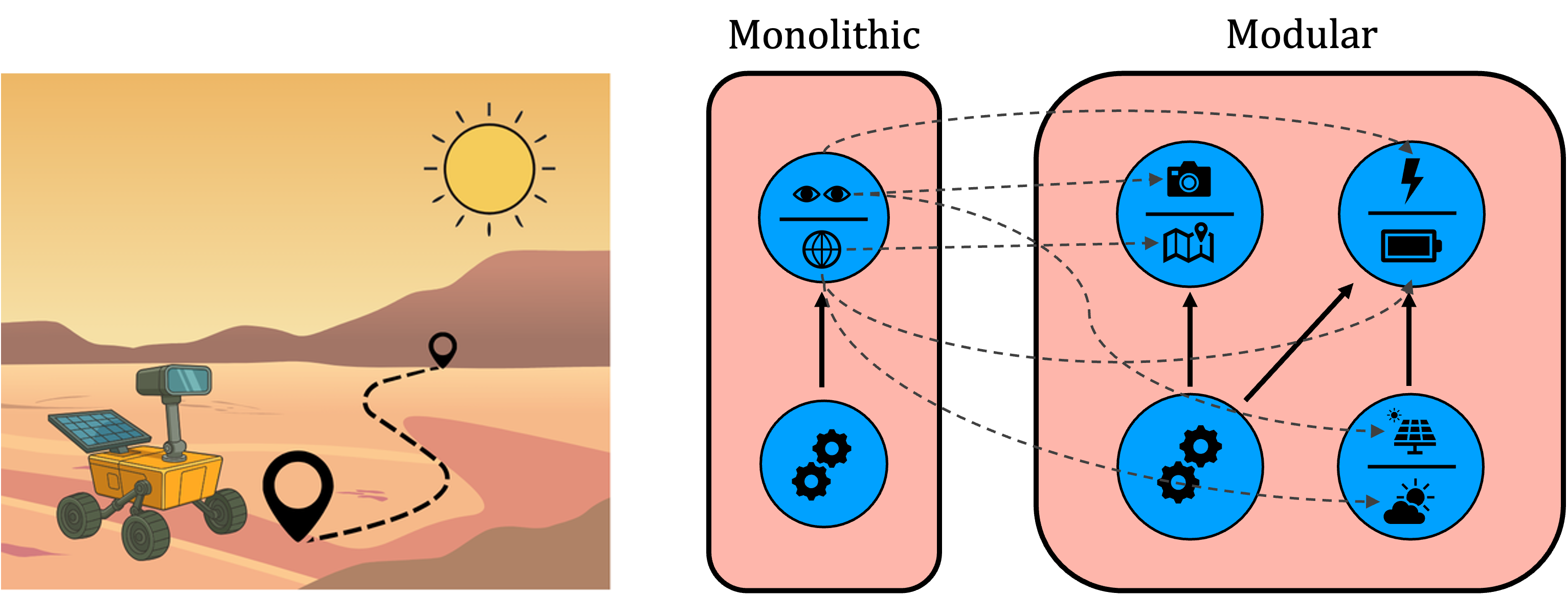}
    \caption{
    \centering
    \small{Two different representations of a world model for the Mars rover example. Middle: a monolithic world model representing the mapping from motor outputs $X$ to observations $Y$ using a single, entangled latent variable $R$. Right: a compositional model representing the mapping in terms of coupled modules, including a motor output module, a camera image model with position latents, a battery voltage model with battery state latents, and a solar radiation model with weather state latents.}}
\end{figure*}
\end{example}

\subsection{Related work on transducer composition}
\label{sec:related_work_trasnducer_comp}
The notion of a transducer network naturally generalises factored MDPs~\citep{boutilier2000stochastic,guestrin2003efficient}, which represent the environment's dynamics via a Bayesian network (see \autoref{ex:gridworld}). Transducer networks also generalise factored POMDPs~\citep{poupart2005exploiting} that extend this to partially observable scenarios, decentralised POMDPs~\citep{bernstein2002complexity} considering multiple agents, action-constrained~\citep{altman2021constrained} and stochastic action-set \citep{boutilier2018planning} MDPs where actions can be masked, contextual MDPs~\citep{hallak2015contextual} whose states include contextual information, and reward machines~\citep{icarte2018using,dohmen2022inferring} considering structured reward functions (see \autoref{ex:multi-objective-transducer}). 
The power of transducer networks is their ability to treat environment physics, agent policies, and task logic as separate dynamical modules that interact through structured input-output subspaces. 
This generality allows them to, for instance, express factored environments with factored rewards, as showcased in \autoref{ex:mars_rover}. 

It is important to note that the study of transducers and their composition has a long history in computational linguistics, based on the key result by~\citet{SCHUTZENBERGER1961185} that the composition of two transducers yields another transducer.
In the linguistics literature, transducers are usually assumed to have a finite set of (latent) states and are thus referred to as finite-state transducers (FSTs).
The transitions between states are weighted by elements from a semiring, and need not be stochastic.
One basic form of composition of FSTs is \emph{parallel composition}, where two FSTs transducing distinct input processes to output processes are combined into a single larger FST~\citep{reape-thompson-1988-parallel, mohri-1997-finite}. 
When the output set of one FST is equal to the input set of another, the two can also be composed \emph{in series} by piping the output of the first to the input of the second, while the output of the second becomes the output of the resulting composed transducer~\citep{mohri2002weighted}. Note that this is an extension of the classic algorithm for intersecting finite state automata (FSTs without outputs) as described by~\citet{hopcroft1979introduction}.
A third form of transducer composition that relates closely to ours appears in the context of algebraic automata theory in the form of \emph{cascade products} of FSTs~\citep{maler2010krohnrhodes, maler1993decomposition}, where the output of the first transducer is its latent state, and the input of the second is the combination of the input and output of the first. This type of composition is commonly used in the decomposition of finite state automata~\citep{krohn-rhodes-1965-algebraic, eilenberg1976automata, maler1993decomposition}.

Our formulation of transducer composition generalizes the aforementioned types of composition, meaning each of them can be embedded in ours (see \autoref{app:Other Composition Types}). 
Furthermore, our definition of transducers does not require the number of states to be finite, whereas we impose that the transitions are stochastic (see \autoref{def:transducer}).

\section{Decomposing transducers through factorization}
\label{sec:decomp_framework}

We now investigate under what conditions the composition of transducers, as established by  \autoref{def:Composition}, can be inverted.  We call the process of inverting composition \emph{factorization}, because it results in expressing the original transducer as a product of ``prime'' components. 

In the following, we first explore the problem of transducer factorisation in scenarios where latent processes are known (\autoref{sec:factoring_latents}), and then address the question of how to do this in situations where latents are not accessible (\autoref{sec:factoring_nolatents}).

\subsection{Factoring with latents}
\label{sec:factoring_latents}

Our starting point for this procedure is a collection of random variables with no predefined causal ordering.  Instead, we just consider random variables categorised into two classes: observable variables, indexed by the set $\mathcal{I}$, and latent variables, indexed by the set $\mathcal{I}'$. Hence, each $i \in \mathcal{I}$ has a corresponding observable random variable $X(i)$, and each $i' \in \mathcal{I}'$ has a corresponding latent random variable $R(i')$. The overall process is represented by a single node $\frac{X(\mathcal{I})}{R(\mathcal{I}')}$, where $X(\mathcal{I})=\{X(i)|i \in \mathcal{I}\}$ and $R(\mathcal{I}')=\{R(i')|i' \in \mathcal{I}'\}$.  The goal is to cluster subsets of $\mathcal{I}$ and $\mathcal{I}'$ into smaller nodes, which will correspond to elementary transducers.  This can be done using our understanding of the statistics of single transducers, as explained next.

\subsubsection{Identifying transduction via Intransducibility}

\begin{figure*}
\centering
\includegraphics[width=.6\columnwidth]{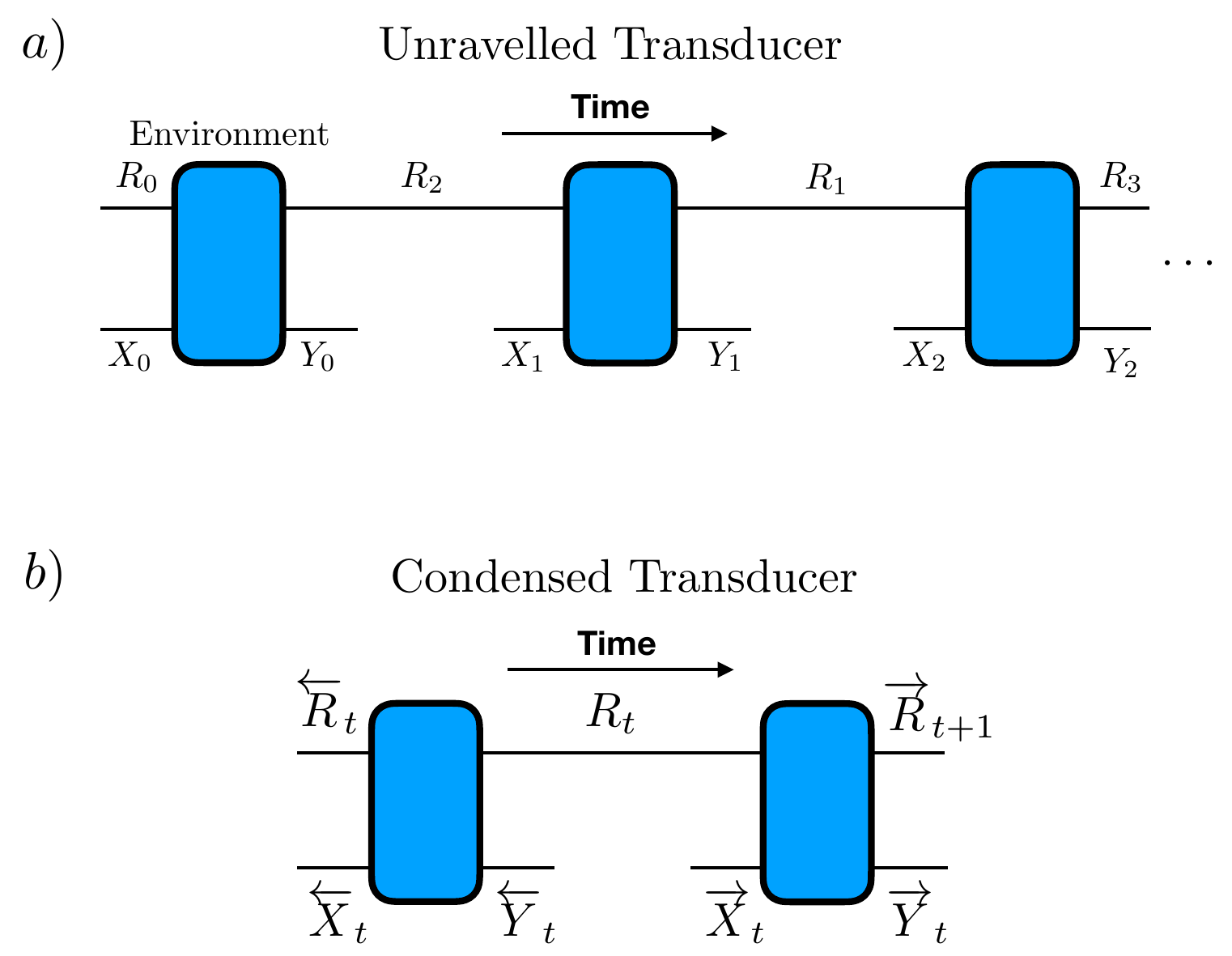}
\caption{\centering\small \emph{Two illustrations of a transducer:} An unravelled transducer (a) shows individual inputs $X_t$, outputs $Y_t$, and latent states $R_t$ unravelled into a semi-infinite sequence.  The same object can be condensed (b) into a mapping from input pasts $\overleftarrow{X}_t$, input futures $\overrightarrow{X}_t$, and latent state pasts $\overleftarrow{R}_t$ to output pasts $\overleftarrow{Y}_t$, output futures $\overrightarrow{Y}_t$, latent state futures $\overrightarrow{R}_t$ via the present latent state $R_t$.}
\label{fig:Condensed_Transducer} 
\end{figure*}

As a first step towards decomposing transducers, we address the following question: given a joint distribution over processes $\Pr(X,Y,R)$, when can $X$ be interpreted as the input of a transducer with latent states $R$ and outputs $Y$?  

To study this, we will use the following shorthand notation:
\begin{itemize}
    \item $X$ \textbf{can} be transduced to $Y$ via $R$:  $X\Rrightarrow Y|R$.
    \item $X$ \textbf{cannot} be transduced to $Y$ via $R$:  $X\not\Rrightarrow Y|R$.
\end{itemize}
We next introduce an information quantity that will play a key role in answering this question. 
\begin{definition}
The \textbf{Intransducibility} of a joint distribution over processes $\Pr(X,Y,R)$ is 
\begin{align}
f(X \Rrightarrow Y|R) \equiv \sum_{t = 0}^{\infty}
I[\overrightarrow{R}_{t+1},\overrightarrow{Y}_t;\overleftarrow{Y}_t,\overleftarrow{R}_t,\overleftarrow{X}_t|\overrightarrow{X}_t,R_t].
\label{eq:transduction_violation}
\end{align}
\end{definition}
Above, $I[\overrightarrow{R}_{t+1},\overrightarrow{Y}_t;\overleftarrow{Y}_t,\overleftarrow{R}_t,\overleftarrow{X}_t|\overrightarrow{X}_t,R_t]$ measures to what degree the past reveals information about future latent states and outputs independent of the present latent state and future inputs.  As illustrated in the architecture of any transducer (shown in 
\autoref{fig:Condensed_Transducer}), the present latent state and future inputs shield past and future, making direct information sharing between them impossible.  Our next result shows that if $X,Y,R$ are the components of a transducer then the Intransducibility needs to be zero (proof in~\autoref{app:when_trans}).
%This leads to a measure that allows us to determine the causal influence between variables in a network from the statistics of the observed and latent variables
\begin{lemma}\label{res:when_trans}
$X$ can be transduced to $Y$ via $R$  ($ X \Rrightarrow Y|R$) if and only if the Intransducibility is zero ($f(X \Rrightarrow Y|R) =0$). 
\end{lemma}

This lemma provides a clear criterion for deciding when a set of processes constitutes a transducer, which will be the basis of the decomposition procedure presented next.  In essence, zero Intransducibility from $X$ to $Y|R$ means that the $Y|R$ node is downstream from $X$ in a feedforward network.

While Intransducibility is a measure designed to identify the validity of input-output channels, it can also be used in the restricted context of generating a process \citep{ellison2011information, boyd2018thermodynamics}.  Rather than addressing a joint process of observables and $Y$ and latents $R$, conditioned on input observables $X$, we can simply consider whether $R$ can function as the latent states of an HMM that generates $Y$.  In this case, the Intransducibility measure reduces to
\begin{align}
f(\varnothing \Rrightarrow Y|R) \equiv \sum_{t = 0}^{\infty}
I[\overrightarrow{R}_{t+1},\overrightarrow{Y}_t;\overleftarrow{Y}_t,\overleftarrow{R}_t|R_t].
\end{align}
$\varnothing \Rrightarrow$ indicates that there is no driving process.  This can be applied to HMM generators, which are effectively input-agnostic transducers:

\begin{corollary}
$Y$ can be produced by an HMM with hidden states $R$ ($\varnothing \Rrightarrow Y|R$) if and only if $f(\varnothing \Rrightarrow Y|R)=0$.
\end{corollary}
This condition requires that $I[\overrightarrow{R}_{t+1},\overrightarrow{Y}_t;\overleftarrow{Y}_t,\overleftarrow{R}_t|R_t]=0$ for all $t$, which further implies that
\begin{align}
I[\overrightarrow{Y}_t;\overleftarrow{Y}_t|R_t]=0.
\end{align}
This means that the memory of the HMM contains all information shared between the past and future outputs \citep{ellison2011information}.  We can use the Intransducibility measure to factor large collections of observable and latent processes into modular subcomponents.

\subsubsection{Factoring high-dimensional processes with latents}

Let us consider a general high-dimensional process, which is described via a collection of observed processes $X(\mathcal{J})=\{X(j) \}_{j \in \mathcal{J}}$ and latent state processes $R(\mathcal{J}')=\{R(j) \}_{j \in \mathcal{J}'}$.  Our starting point is the joint probability distribution $\Pr(X(\mathcal{J}),R(\mathcal{J}'))$, which is accessible either through an analytical model or sufficient data to estimate mutual informations over time.   $\mathcal{J}$ and $\mathcal{J}'$ are finite sets that index the observable and latent processes, respectively.
Each element $X(j)$ is the random variable for a probability distribution over observable words, like the data one may take from outputs of a transformer, and each $R(j)$ is the random variable for a probability distribution over latent words, like the residual stream of the same transformer.  One might imagine that a larger environment is composed of many such transformers, feeding into each other, and our task is to determine how information flows between these different elements of this complex world.  The algorithm described here provides a strategy for determining this network of dependencies, paralleling the process of factoring a large integer into primes.  In this process, the measure of Intransducibility we introduced in the last section functions like the remainder of a division operation, indicating how far we are from an exact product of terms.  The result of the factorization algorithm will be an ordered set of nodes, each labeled with a subset of the observable processes and a (possibly null) subset of the latent processes.  The ordering of the nodes is a viable causal ordering of the processes.

In factorizing a joint process $\Pr(X(\mathcal{J}),R(\mathcal{J'}))$ of observables and latents, it is useful to define prime processes, which cannot be decomposed.  For subsets of the observable and latent spaces, $\mathcal{O} \subseteq \mathcal{J}$ and $\mathcal{O}' \subseteq \mathcal{J}'$ respectively, we can use the Intransducibility to determine whether the subset of observable processes $X(\mathcal{J}-\mathcal{O})$ can be interpreted as the input for a transducer that outputs $X(\mathcal{O})$ with latent memory $R(\mathcal{O}')$, meaning $X(\mathcal{J}-\mathcal{O}) \Rrightarrow X(\mathcal{O})|R(\mathcal{O}') $.  In addition, we must check whether the processes $R(\mathcal{J}'-\mathcal{O}')$ can be interpreted as the latent process of an HMM generator of $X(\mathcal{J}-\mathcal{O})$, meaning $\varnothing \Rrightarrow X(\mathcal{J}-\mathcal{O})|R(\mathcal{J}'-\mathcal{O}')$.
\begin{definition}
The joint set of observables and latents $(\mathcal{J}, \mathcal{J}')$ is \textbf{prime} if there does not exist a pair of subsets $\mathcal{O} \subseteq \mathcal{J}$ and $\mathcal{O}' \subseteq \mathcal{J}'$ such that the observable subset is nonempty $\mathcal{O} \neq \emptyset$, $X(\mathcal{J}-\mathcal{O}) \Rrightarrow X(\mathcal{O})|R(\mathcal{O}') $, and $\varnothing \Rrightarrow X(\mathcal{J}-\mathcal{O})|R(\mathcal{J}'-\mathcal{O}')$.  
\end{definition}

\autoref{fig:Recursive_Decomposition} shows a recursive process that decomposes a large transducer into smaller prime elements, and Algorithm \autoref{alg:Transduction Violation} provides the pseudocode\footnote{In general, additional assumptions such as a finite horizon, convergence of the infinite series in Equation~\ref{eq:transduction_violation}, or stationarity would be required for the algorithms to be computable.}. The input to the algorithm is the joint process over observables and latents $\Pr(X(\mathcal{J}),R(\mathcal{J}'))$, and the sets of observables and latents themselves $(\mathcal{J},\mathcal{J}')$.  The output $\mathcal{M}$ is an ordered list of subsets, where the $n$th element of $\mathcal{M}$ is $(\mathcal{O}(n),\mathcal{O}'(n))$, where $\mathcal{O}(n) \subseteq \mathcal{J}$ and $\mathcal{O}'(n) \subseteq \mathcal{J}'$, and specifies the node $\frac{X(\mathcal{O}(n))}{R(\mathcal{O}'(n))}$ in the dependency graph.  By construction, this node can only be directly influenced by the node $\frac{X(\mathcal{O}(n'))}{R(\mathcal{O}'(n'))}$ if $n' < n$, which reconstructs the ordering that is consistent with our original composition of transducers.  The process requires finding the smallest set of observable variables and latent states (characterized by the sets $\mathcal{O}$ and $\mathcal{O}'$) such that the remaining observables (characterized by $\mathcal{J}-\mathcal{O}$) are inputs to the causal channel $X(\mathcal{J}-\mathcal{O}) \Rrightarrow X(\mathcal{O})|R(\mathcal{O}') $.  If we apply this test recursively, we will discover a feedforward network of dependencies of the transducer. However, this decomposition may not be unique in general, as there may be several smallest sets of variables satisfying the above condition for a given joint process. Thus, depending on how ties are broken in selecting this set, the resulting decompositions may be different.

To illustrate this algorithm with an example, if we are given the data for a transducer $\frac{X(0:N)}{R(0:N)}$ where the elements are composed $\frac{X(n)}{R(n)}$ with $n \in \{0,\ldots, N-1\}$ in ascending order, applying the algorithm would first identify the last element $\frac{X(N-1)}{R(N-1)}$ as the smallest sets of random variables that can be transduced from the remaining observables.

\begin{figure*}
\centering
\includegraphics[width=\columnwidth]{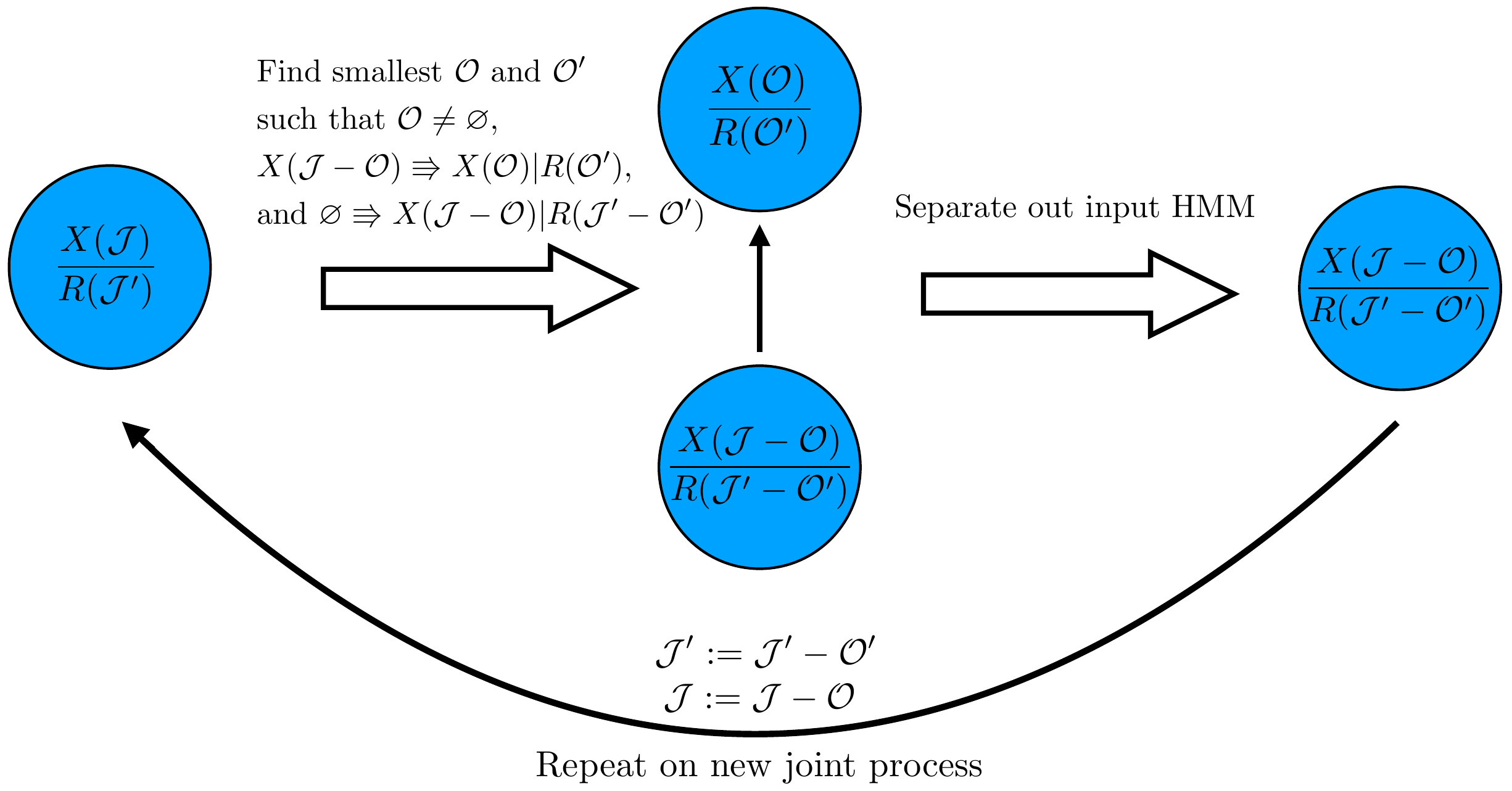}
\caption{
\centering
\small{By finding the smallest sets of observables $\mathcal{O}$ and latents $\mathcal{O}'$ such that the remaining observables $\mathcal{J}-\mathcal{O}$ function as the input to a transducer with output $X(\mathcal{O})$ and latents $R(\mathcal{O}')$, we can decompose the transducer network.  This provides a factorization of the overall transducer $\frac{X(\mathcal{J})}{R(\mathcal{J}')}$. }}
\label{fig:Recursive_Decomposition} 
\end{figure*}

\begin{algorithm}[H]
\label{alg:Transduction Violation}
\caption{Decomposition via Intransducibility}
$P :=  \Pr(X(\mathcal{J}),R(\mathcal{J}'))$ \;
$J_{\mathrm{obs}} := \mathcal{J}$ \;
$J_{\mathrm{lat}} := \mathcal{J}'$ \;
$\mathcal{M} := \emptyset$ \;

\While{$(J_{\mathrm{obs}}, J_{\mathrm{lat}})$ is not prime }{
    Find one of the smallest $(\mathcal{O},\mathcal{O}')$ such that $\mathcal{O} \neq \emptyset$, 
    
    $f(X(J_{\mathrm{obs}}- \mathcal{O}) \Rrightarrow X(\mathcal{O}),
    \mid R(\mathcal{O}')) = 0$, 
    
    and $f(\varnothing \Rrightarrow X(J_{\mathrm{obs}}- \mathcal{O})|R(J_{\mathrm{lat}}- \mathcal{O}')) = 0$ \;
    (Note: Intransducibility $f$ is evaluated using the joint process $P$) \;

    Prepend module $( \mathcal{O},\; \mathcal{O}')$ to $\mathcal{M}$ \;

        $J_{\mathrm{obs}} := J_{\mathrm{obs}} - \mathcal{O}$ \;
        $J_{\mathrm{lat}} := J_{\mathrm{lat}} - \mathcal{O}'$ \;

}
Prepend $(J_{\mathrm{obs}}, J_{\mathrm{lat}})$ to $\mathcal{M}$ \;
\Return{$\mathcal{M}$} \;

\end{algorithm}
At each step, this algorithm identifies one of the smallest downstream observable-latent combinations $(\mathcal{O},\mathcal{O}')$, and peels it off as its own modular node.

\begin{example}
    Extending the running Mars rover example, consider an additional thermal
    module that tracks the internal temperature $\Theta_t$ of the rover using a thermometer reading $\tilde{\Theta}_t$, which is dependent on the ambient temperature $A_t$ and a heating system $H_t$. The rover's components may not function properly if the internal temperature drops below a critical threshold, requiring the constant monitoring and regulation of this variable. A monolithic world model would entangle thermal dynamics of the rover with navigation. However, applying the
    Intransducibility metric allows us to discover modular structure.

    Suppose we have access to both observables
    $(X, H, I, L, A, \tilde{\Theta}, V)$ and latent states $(P, W, B, \Theta)$,
    where $X$ and $H$ denote motor and heater commands respectively, $I$ is the camera image, $L$ is illuminance, $A$ is ambient temperature, $\tilde{\Theta}$ is the measured internal temperature, $V$ is battery voltage, and the latents represent position, weather state, battery state, and internal temperature, respectively.

    Applying Algorithm~\ref{alg:Transduction Violation}, we search for the smallest subsets $(\mathcal{O}, \mathcal{O}')$ that satisfy the condition in the algorithm, which proceeds as follows:

	    \begin{enumerate}
	        \item \textbf{Iteration 1}: Find that $\mathcal{O} = \{I\}$ and
	              $\mathcal{O}' = \{P\}$ satisfy
	              $f((X, H, L, A, \tilde{\Theta}, V) \Rrightarrow I \mid P) = 0$.
	              Camera images depend on motor commands only through position. Heater commands, weather, and thermal state do not affect the terrain that the rover sees. This is the smallest extractable subset, so we extract module $T_{\mathrm{nav}}$.

	        \item \textbf{Iteration 2}: From the remaining observables
	              $(X, H, L, A, \tilde{\Theta}, V)$ and latents $(W, B, \Theta)$,
	              testing whether power or thermal can be extracted independently:
	              \begin{itemize}
	                  \item $f((X, H, L, A, \tilde{\Theta}) \Rrightarrow V \mid B) > 0$
	                        because voltage predictions improve when conditioning on thermal state (since efficiency depends on temperature).
	                  \item $f((X, H, L, A, V) \Rrightarrow \tilde{\Theta} \mid \Theta) > 0$
	                        because temperature predictions improve when conditioning on battery state (since heater operation depends on available power).
	              \end{itemize}
	              However, jointly we find $\mathcal{O} = \{V, \tilde{\Theta}\}$ and $\mathcal{O}' = \{B, \Theta\}$ satisfy $f((X, H, L, A) \Rrightarrow (V, \tilde{\Theta}) \mid (B, \Theta)) = 0$. The power and thermal subsystems form a \emph{coupled module} $T_{\mathrm{pwr{+}therm}}$ that cannot be further factored.

	        \item \textbf{Iteration 3}: The remaining observables are $(X, H, L, A)$ with latent $(W)$. We find that $\mathcal{O} = \{L, A\}$ and $\mathcal{O}' = \{W\}$ satisfy $f((X, H) \Rrightarrow (L, A) \mid W) = 0$. The weather process is input-agnostic: solar irradiance and ambient temperature depend only on atmospheric dynamics, not rover commands. We therefore extract module $T_{\mathrm{sol}}$ to arrive at the final decomposition.
	    \end{enumerate}

    The algorithm thus recovers a four-module decomposition: an input motor and heating module, a navigation module, an exogenous weather module, and a coupled power-thermal module.
\end{example}

\begin{figure}[t]
    \centering
    \includegraphics[width=\columnwidth]{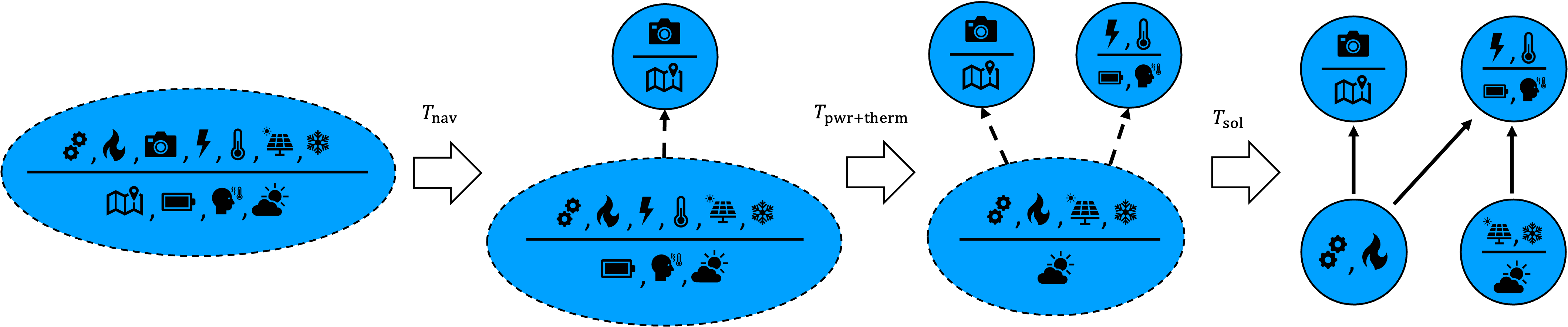}
    \caption{
    \small{Iterative decomposition of the Mars rover world model. Starting from the monolithic transducer $T_{\mathrm{world}}$ (left), the algorithm first extracts the navigation module $T_{\mathrm{nav}}$ by identifying that camera images $I$ depend on motor commands only through position $P$. Next, the exogenous weather module $T_{\mathrm{sol}}$ is separated, since solar irradiance and ambient temperature are input-agnostic. The remaining power and thermal subsystems cannot be factored independently and form the coupled module $T_{\mathrm{pwr{+}therm}}$. The final network (right) shows the three downstream modules of the decomposition (with the input motor/heater commands acting as the source) with causal dependencies: weather and inputs jointly drive the power-thermal module, while inputs alone drive navigation.}}
    \label{fig:rover-decomposition}
    \end{figure}

\subsection{Factoring without latents}
\label{sec:factoring_nolatents}

In the absence of explicit latent variables, we can no longer appeal to Intransducibility directly, but we can still exploit the nonanticipatory constraint to probe causal structure from observables alone. To do so, we now introduce an observable-only diagnostic (Acausality) which quantifies how much an interface deviates from being realizable as a feedforward transducer.

\subsubsection{Measuring dependency without latents}

When viewing the behavior of an environment, we may not have access to the latent parameters that carry information between inputs and outputs.  In this case, it is still possible to determine a causal architecture for the data using the constraints of transduction.  
As shown in \autoref{thm:transducer_interface_equivalence}, being nonanticipatory (i.e., an interface) is a necessary and sufficient condition for the joint process $\Pr(X,Y)$ to be describable via the application of a transducer. 
Therefore, for two processes $X$ and $Y$, we can introduce a measure which tells us whether $X$ could be the input to a transducer that produces outputs $Y$, a condition we denote $X \Rrightarrow Y$.

We define the \emph{Acausality} of an interface $\mathcal{I}[Y|X]$, which is a measure that allows us to determine the degree to which it violates the nonanticipatory condition. 
\begin{definition}
    The \textbf{Acausality} of an interface $\mathcal{I}[Y|X]$ is the sum over all time steps of the mutual information between past outputs and future inputs conditioned on past inputs
    \begin{align}
AC[X \Rrightarrow Y]=\sum_{t=0}^{\infty}I[\overleftarrow{Y}_t;\overrightarrow{X}_t|\overleftarrow{X}_t].
    \end{align}
\end{definition}
Informally, Acausality is large when past outputs of the interface depend on future inputs, even after conditioning on past inputs. In other words, it quantifies how much the interface “looks like it needs to see the future” to explain its behavior. 
This measure can be used to determine whether the flow of information can be interpreted as proceeding from $X$ to $Y$

\begin{lemma}
The interface $\mathcal{I}[Y|X]$ can be expressed as a transducer from $X$ to $Y$ iff the Acausality is zero $AC[X \Rrightarrow Y]=0$.
\end{lemma}
\begin{proof}
Let us prove in both directions:
\begin{itemize}
    \item $\Rightarrow$ If $\mathcal{I}[Y|X]$ can be expressed as a transducer, then it is nonanticipatory, and $\Pr(\overleftarrow{Y}_t|\overrightarrow{X}_t,\overleftarrow{X}_t)=\Pr(\overleftarrow{Y}_t|\overleftarrow{X}_t)$, which further implies that $I[\overleftarrow{Y}_t;\overrightarrow{X}_t|\overleftarrow{X}_t]=0$ for all $t$.  Therefore, all elements of the sum of the Acausality are zero, and the Acausality is zero.
    \item $\Leftarrow$ Because conditional mutual information is always greater than or equal to zero, $I[A;B|C]\geq 0$, this implies that all elements of the sum of Acausality must be zero
    \begin{align}
        I[\overleftarrow{Y}_t;\overrightarrow{X}_t|\overleftarrow{X}_t]=0 \text{ for all }t.
    \end{align}
    The conditional mutual information $I[A;B|C]$ is zero iff $\Pr(A|BC)=\Pr(A|C)$, so we arrive at the condition for being nonanticipatory $\Pr(\overleftarrow{Y}_t|\overrightarrow{X}_t,\overleftarrow{X}_t)=\Pr(\overleftarrow{Y}_t|\overleftarrow{X}_t)$.  Thus, the interface can be expressed as a transducer from $X$ to $Y$.
\end{itemize} 
\end{proof}

The Acausality $AC[X \Rrightarrow Y]$ of the interface is zero iff $X \Rrightarrow Y$.  Therefore, there exists a latent variable $R$ that can transform $X$ into $Y$.  A strong candidate is the set of causal states of the $\epsilon$-transducer~\citep{barnett2015computational}.  As before, we can apply this simple measure recursively to determine the dependency between many observed variables. 

% \begin{example}
% Suppose we are given only the observable data from the rover: inputs $X$ (motors) and outputs comprising solar sensor readings $S$. We wish to determine if the weather system can be treated as an independent module by computing the Acausality metric $AC[X \Rrightarrow S]$. Since the weather is an exogenous process governed by atmospheric physics independent of the rover's movements, the history of solar radiation $\overleftarrow{S}_t$ contains no information about future motor commands $\overrightarrow{X}_t$ that is not already in $\overleftarrow{X}_t$. Thus, $AC[X \Rrightarrow S] = 0$.
% This mathematically justifies decomposing the weather into a separate, input-agnostic module ($T_{\mathrm{sol}}$). Practically, this allows for a modular approach to training: one can train a navigation policy on $T_{\mathrm{nav}}$ alone (ignoring weather noise) and an energy policy on $T_{\mathrm{pwr}}$ composed with $T_{\mathrm{sol}}$, then recompose them for the final mission.
% \end{example}

\subsubsection{Factoring high-dimensional processes without latents}

The measure of Acausality allows us to build a network of dependencies between random variables $X(\mathcal{J})$, where there is an edge from node $j$ to $j'$ iff $X(j) \Rrightarrow X(j')$. Once again, it is useful to define primes in the context of an observable process
\begin{definition}
The set of observables $\mathcal{J}$ is \textbf{prime} if there does not exist a subset of the observables $\mathcal{O} \subseteq \mathcal{J}$ such that the observable subset is nonempty $\mathcal{O} \neq \emptyset$ and $X(\mathcal{J}-\mathcal{O}) \Rrightarrow X(\mathcal{O}) $.
\end{definition}
The Acausality allows us to factor a process composed of many elements in the same way as before, functioning as the remainder in the ``division'' operation.

Algorithm \autoref{alg:Acausality} shows the general method for discovering such a network given the probabilities $\Pr(X(\mathcal{J}))$ of that set of random variables.  It effectively mirrors the strategy when the probability of latent variables $R(\mathcal{J}')$ is also known, but reduces the difficulty of those calculations.  Once you have factored the set of observable random variables, inference becomes simpler.  Rather than infer the $\epsilon$-transducer from the \emph{whole system}, we can infer the causal structure of subcomponents of the factorized transducer.  This is sufficient to reconstruct the overall causal architecture of the whole system.

\begin{algorithm}[H]
\label{alg:Acausality}
\caption{Decomposition via Acausality}
$P :=  \Pr(X(\mathcal{J}))$ \;
$J_{\mathrm{obs}} := \mathcal{J}$ \;
$\mathcal{M} := \emptyset$ \;

\While{$J_{\mathrm{obs}}$ is not prime }{
    Find one of the smallest $\mathcal{O} $ such that $\mathcal{O} \neq \emptyset$ and
    $AC[X(J_{\mathrm{obs}}- \mathcal{O}) \Rrightarrow X(\mathcal{O}) ] = 0$ \;
    (Note: Acausality $AC$ is evaluated using the joint process $P$) \;

    Prepend module $\mathcal{O}$ to $\mathcal{M}$ \;

        $J_{\mathrm{obs}} := J_{\mathrm{obs}} - \mathcal{O}$ \;

}
Prepend $J_{\mathrm{obs}}$ to $\mathcal{M}$ \;
\Return{$\mathcal{M}$} \;

\end{algorithm}

\subsection{Relation to causal discovery}
\label{sec:causal_discovery}
Our Acausality-based factorization relates to a broader literature on discovering causal links between subsystems in dynamical processes~\citep{glymour2019review,nogueira2022methods}. Time-series causal discovery methods such as PCMCI+ \citep{runge2019detecting} estimate directed networks by combining conditional independence tests with graph search, under time-series Markov and faithfulness assumptions. In contrast, we start from a representation of the agent–environment interface and use an information-theoretic notion of Acausality—essentially directed information from future inputs to past outputs—to decide when that interface can be realized as a feedforward transducer and how it factors into composable modules. Our method is capable of detecting non-Markovian dependencies between time series.

Our approach is also complementary to recent information-theoretic causality measures such as SURD \citep{martinez2024decomposing,jansma2025decomposing}, which decompose the strength of multi-variable causal influence into unique, redundant, and synergistic contributions: Acausality provides a structural test for when a module can be treated as causally downstream of another, while SURD-style decompositions can, in principle, be applied within each module to analyze the nature of multi-parent influences. Finally, empirical work on compositional neural subspaces in cortex \citep{tafazoli2025building} suggests that biological agents may realize world-models as compositions of reusable latent subspaces; our results provide a normative framework for such ‘cognitive building blocks’ in terms of composable transducers.

\section{Coarse-graining networks of transducers}
\label{sec:coarse-graining}

Decomposition recovers a fine-grained network of interacting transducers whose joint operation generates a high-dimensional observable process. Once this network has been identified, one might well want to operate at a reduced level of description --- e.g. eliminating variables that are not relevant for a given predictive or inferential task via coarse-graining. 

Coarse-graining, in this setting, refers to reducing the state space of the overall transducer by removing nodes entirely: their observable processes and the latent states required to model them. This operation must preserve the marginal interface among the remaining nodes. Thus, coarse-graining is only valid when the nodes being removed can be `lumped out' without altering the conditional distribution governing the surviving observables. This extends the idea of  `lumpability' from Markov chains~\citep[Ch.~6.3]{kemeny1969finite} to multi-transducer networks with both inputs and outputs.

\subsection{A multiscale perspective on transducer networks}

A network of $N$ transducers may be viewed as a single input-independent transducer with joint latent state $R(0:N)$, generating the observable process $\Pr(X(0:N))$ (see \autoref{fig:N_Observable_Dependencies}).  From this perspective, different descriptive scales correspond to retaining only a subset of nodes and eliminating the rest while preserving the marginal interface on the nodes we keep.

Consider selecting a contiguous block of nodes composed of the transducers $T(n),T(n+1), \cdots, T(n+a-1)$, whose observables $X(n:n+a)$ and latents $R(n:n+a)$ we wish to track explicitly.  We call such a block causally adjacent when every node outside it lies entirely upstream or entirely downstream.  In this case, all predictive influence from the eliminated nodes reaches the retained block only through its boundary observables, ensuring that those external nodes can be coarse-grained without altering the conditional dynamics of the block.

This viewpoint provides a clean multiscale structure: once a block has been identified, nodes before it may be eliminated by conditioning on their observables, and nodes after it may be eliminated by marginalizing over theirs.  These coarse-grainings are consistent because they preserve the marginalized interface.  In other words, we can ``lump'' latent states together as long as they are consistent within this causally adjacent chain.  The following subsections detail these two forms of coarse-graining and their role in reducing the state space of the transducer network.  This is useful in extrapolating less complex structures from the overall network.

\begin{figure*}[h!]
\centering
\includegraphics[width=.9\columnwidth]{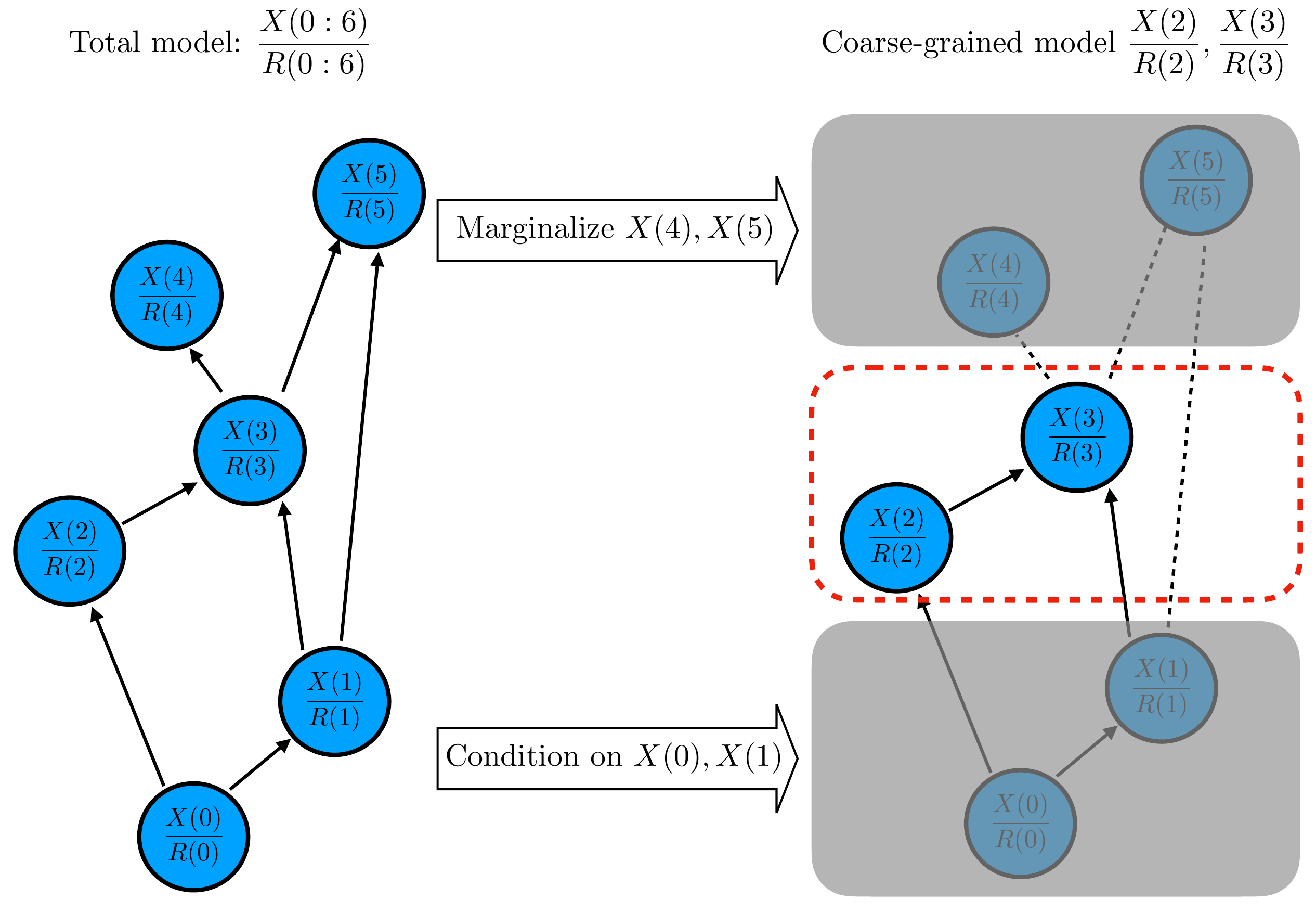}
\caption{\centering\small \emph{The total transducer network can be simplified from the bottom by conditioning on upstream observables, and from the top by marginalizing downstream observables.}}
\label{fig:Coarse_Grain} 
\end{figure*}

\subsubsection{Simplifying from the top}

If our objective is to describe the system only up to $X(b-1)$, then all nodes $X(b:N)/R(b:N)$ lie strictly downstream of the region of interest.  Their influence on earlier nodes occurs solely through the observable variables we choose not to retain.  Thus, we may marginalize over these downstream observables:
\begin{align}
    \mathcal{I}[X(0:b)|\emptyset]= \sum_{x(b:N)}\Pr(X(0:b),X(b:N)=x(b:N)).
\end{align}
Marginalizing over $X(b:N)$ also removes any dependence on the downstream latent variables $R(b:N)$, whose only predictive role is mediated through the eliminated observables.  This yields a reduced transducer representation that exactly preserves the interface on the retained variables.

\subsubsection{Simplifying from the bottom}

Consider a contiguous block of nodes $X(0:a)/R(0:a)$ forming the bottom of the network.  Even though individual nodes within the block may have internal dependencies, the block as a whole has no incoming edges from the outside.  Therefore, its joint observable trajectory $X(0:a)$ can be treated as an externally supplied input.

Conditioning on these observables removes any dependence on their latent states:
\begin{align}
\mathcal{I}[X(a:b)|X(0:a)]= \frac{\Pr(X(0:b))}{\Pr(X(0:a))}.
\end{align}
Once the entire block's upstream observables are fixed, the latent variables $R(0:a)$ make no further predictive distinctions for downstream nodes.  They can therefore be coarse-grained out, yielding an exact reduced complexity representation of the system's behavior above the block.

\begin{figure}[t]
    \centering
    \includegraphics[width=\columnwidth]{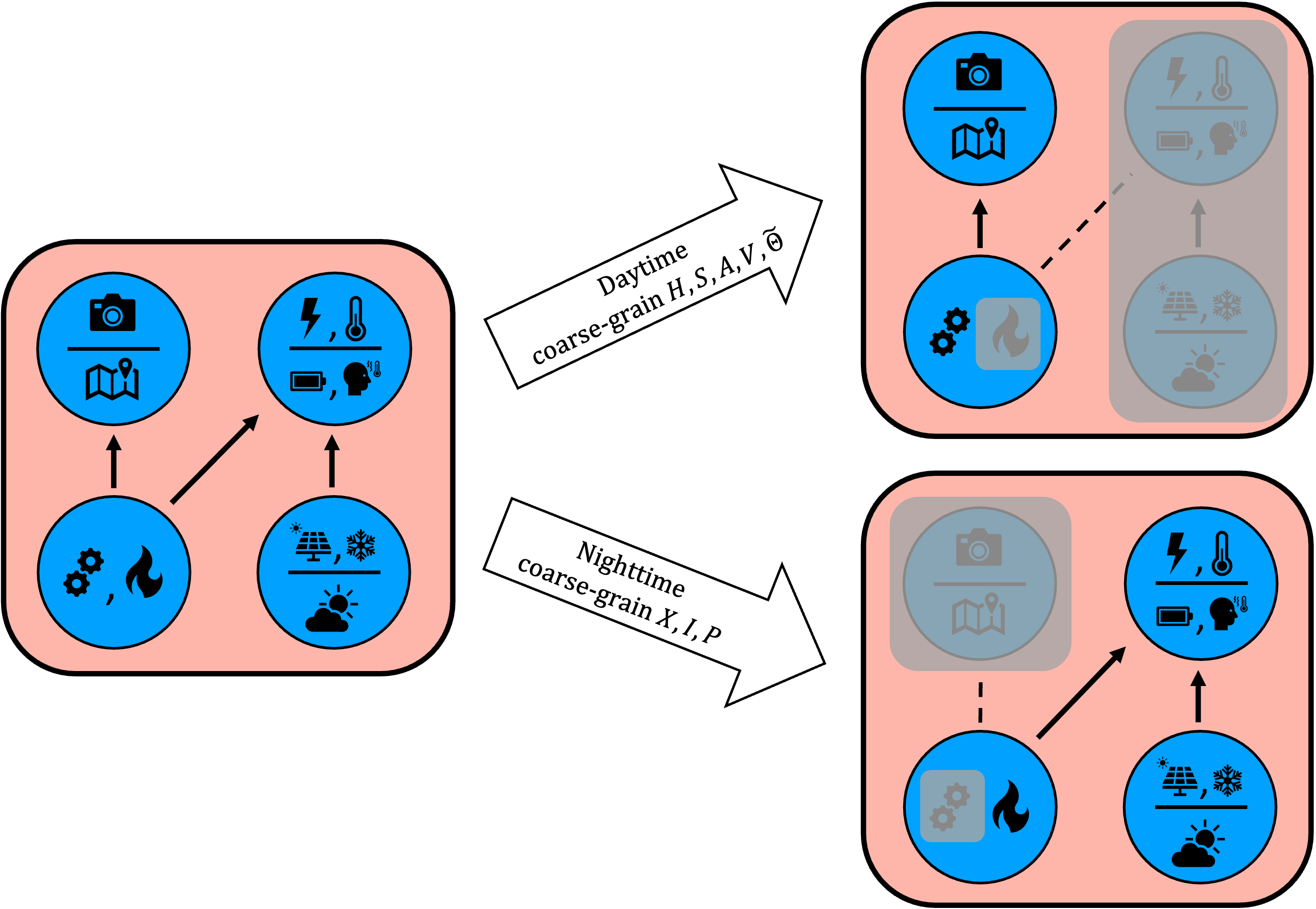}
    \caption{Mission-dependent coarse-graining of the decomposed rover network. The full three-module network (left) is simplified under two scenarios. 
    \emph{Daytime} (top right): the mission objective depends only on position, so marginalizing over $(H, L, A, V, \tilde{\Theta})$ eliminates the weather and power-thermal modules, reducing the world model to the navigation transducer $T_{\mathrm{nav}}$ alone. \emph{Nighttime} (bottom right): the rover is stationary and the objective shifts to thermal survival, so conditioning on fixed position eliminates navigation; the reduced model retains only the weather module $T_{\mathrm{sol}}$ and the coupled power-thermal module $T_{\mathrm{pwr{+}therm}}$, capturing the essential trade-off between heater activation and battery depletion.}
    \label{fig:rover-coarse-graining}
    \end{figure}

\begin{example}
\label{ex:rover_daytime}
Returning to our running Mars rover example, the decomposed rover model admits different coarse-grainings depending on the mission objective. We illustrate this with two scenarios that highlight how goal structure determines which modules can be eliminated.

During Martian daytime with clear skies, the illuminance $L$ is high, ambient temperature $A$ is moderate, and the rover's thermal and power states are far from critical thresholds. The mission objective is to reach a waypoint, with reward $r_t = -d(P_t, P_{\mathrm{goal}})$ depending only on position.

In this regime, the weather, thermal, and power modules are not reward-relevant: battery drain from motors is easily replenished, and temperature regulation requires no active management. We can therefore coarse-grain by marginalizing over $(H, L, A, \tilde{\Theta}, V)$:
\begin{align}
    \mathcal{I}[I \mid X] = \sum_{h, l, a, \tilde{\theta}, v}
    \Pr(I, H{=}h, L{=}l, A{=}a,
    \tilde{\Theta}{=}\tilde{\theta}, V{=}v \mid X).
\end{align}
The effective world model reduces to the navigation transducer $T_{\mathrm{nav}}$ alone, with state space collapsing from $(P, W, B, \Theta)$ to just $P$. This enables training a navigation policy on a significantly simplified simulator that ignores energy dynamics entirely.

During Martian night, illuminance drops to $L \approx 0$ and ambient temperature falls significantly. The rover remains stationary to conserve energy, and the mission objective shifts to thermal survival.

In this case, navigation becomes irrelevant so the rover does not move, and camera images $I$ provide no reward-relevant information. We can therefore coarse-grain by conditioning on a fixed position and marginalizing over $(X, I)$
\begin{align}
    \mathcal{I}[(\tilde{\Theta}, V) \mid H, L, A] =
    \frac{\Pr(\tilde{\Theta}, V, I \mid X{=}0, H, L, A)}
         {\Pr(I \mid X{=}0)}.
\end{align}
The effective latent state space reduces to $(W, B, \Theta)$, eliminating position $P$. The remaining model retains the weather module $T_{\mathrm{sol}}$ and the coupled power-thermal module $T_{\mathrm{pwr{+}therm}}$, capturing the essential trade-off: heater activation $H$ maintains temperature but depletes the battery, which cannot recharge until dawn. An optimal policy must balance immediate thermal reward against the risk of battery depletion.

Importantly, we cannot further coarse-grain by removing power, as the thermal reward depends on heater usage, which is constrained by battery state. Formally, marginalising $V$ would lose the information that $\Pr(H_t \mid \overleftarrow{V}_t)$, i.e., the agent's heating decisions, depend on voltage history. The power and thermal modules remain coupled even after eliminating navigation.
\end{example}

\subsection{Simplifying sparse networks}

In a fully connected network, coarse-graining is only possible at the very top or bottom of the system.  Sparse networks, however, contain many internal clusters that can be simplified.

If a group of nodes has no incoming edges from outside, we may condition on its observables and remove its latent states, as in bottom simplification. If it has no outgoing edges, we may marginalize its observables and remove its latents, as in top simplification.

Both operations preserve the marginal interface on the remaining nodes.  This provides a framework for lumpability in transducers \citep{rosas2024software}: a cluster may be removed whenever it introduces no additional predictive distinctions for the rest of the network.  Sparse dependency structure, therefore, provides many more opportunities to reduce state-space complexity.

\subsection{Parallelizable inference}

Coarse-graining simplifies the latent space, but parallelizable inference relies only on observable processes.  After applying Algorithm \autoref{alg:Acausality}, the system is decomposed into an ordered set of modules, each with its own interface:
\begin{align}
\mathcal{I}[X(n)|X(0:n)].
\end{align}
Because these modules have zero Acausality, each one can be inferred independently from data.  This breaks the global modeling task into a set of smaller sub-problems, each corresponding to a different component of the decomposed observable process.

Thus, even without access to latent variables, decomposition yields a complete set of structurally independent pieces.  Each can be learned in parallel, providing a direct way to build models from coarse-grained observations while avoiding the complexity of inferring a single high-dimensional latent representation.

\section{Decomposing causal states}
\label{sec:decomposition_compmech}

In previous sections, we showed how nonanticipation and sparsity jointly induce hierarchical decompositions of dynamical processes: first by factoring monolithic interfaces into networks of transducers (\autoref{sec:decomp_framework}), and then by moving up and down these networks to obtain emergent macroscopic levels (\autoref{sec:coarse-graining}). 
In this section, we apply this approach to decompose minimal predictive representations of a process, known as $\epsilon$-machines and $\epsilon$-transducers. The $\epsilon$-machine~\citep{crutchfield2012between}, a special case of the $\epsilon$-transducer~\citep{barnett2015computational}, is a powerful tool for mechanistic interpretability which captures the minimal `causal states' for a process, providing a theoretical prediction for what \emph{should} be stored in the latent states of an optimal predictor \citep{marzen2017nearly}. 
Interestingly, it has been shown (see~\citep{rosasai}) that $\epsilon$-transducers are equivalent to predictive state representations~\citep{littman2001predictive}, a reinforcement learning technique to generate a minimal set of state variables for a given partially-observed setting. 
Moreover, recent studies are revealing that causal state structures are reconstructed inside the latent space of neural networks~\citep{shai2025transformers,piotrowski2025constrained,shai2026transformers}.

We therefore investigate how the decomposition principles developed earlier act on these canonical representations: when and how can a monolithic minimal representation be factorized into interacting modules, and what kinds of hierarchies arise from this?

To study this, let us first note that the composition of the two transducers $T$ and $U$ corresponds to the composition of the interface $\mathcal{I}[Z|XY]$ with the interface $\mathcal{I}[Y|X]$ to generate the interface $\mathcal{I}[YZ|X]$.  If both $T$ and $U$ are the $\epsilon$-transducers (minimal predictive model) for their respective interfaces, then they are both unifilar~\citep{barnett2015computational,rosasai}.  Unifilarity means that the next latent variable is a function of the present latent variable, inputs, and outputs
\begin{align}
    R_{t+1}& =\epsilon_T(R_t,Y_t|X_t),
    \\ S_{t+1}& =\epsilon_U(S_t,Z_t|X_t,Y_t).
\end{align}
Here $\epsilon_T$ and $\epsilon_U$ are the $\epsilon$-maps for the two transducers $T$ and $U$ respectively, which are the functions that take present inputs, outputs, and latent variables, to the next latent variable.  This also implies that the joint transducer is unifilar with the $\epsilon$-map for the composed transducer $UT$
\begin{align}
    [R_{t+1},S_{t+1}]& =\epsilon_{UT}([R_{t},S_t],Y_t,Z_t|X_t)
    \\ & =[\epsilon_T(R_t,Y_t|X_t),\epsilon_U(S_t,Z_t|X_t,Y_t)],
\end{align}
as well as an $\epsilon$-mapping from histories to states of the composite machine:
\begin{align}
    [R_t,S_t]& =\epsilon_{UT}(\overleftarrow{Y}_t,\overleftarrow{Z}_t|\overleftarrow{X}_t)
    \\ & =[\epsilon_{T}(\overleftarrow{Y}_t|\overleftarrow{X}_t),\epsilon_{U}(\overleftarrow{Z}_t|\overleftarrow{X}_t,\overleftarrow{Y}_t)].
\end{align}

\begin{theorem}
\label{thm: Composed Causal States}
    The $\epsilon$-transducer of the composite interface $\mathcal{I}[YZ|X]$ of two interfaces $\mathcal{I}[Y|X]$ and $\mathcal{I}[Z|YX]$ is the composition of the $\epsilon$-transducers of those two interfaces.
\end{theorem}

\begin{proof}
It can be shown that the product $\epsilon$-transducer only distinguishes states when they have different predictions.  See \autoref{app:Proving Composite Causal States} for details.  Thus, the product $\epsilon$-transducer is both predictive (a function of the past) and minimal, making its states equivalent to the causal states of the composite interface.
\end{proof}

If the building-block models of a complex graph of interfaces are each minimal predictors, then so too is the composite model.  This guarantees closure of minimal predictive transducers under composition, allowing us to build larger interfaces from smaller ones while preserving minimality.  This also allows us to consider how we might \emph{decompose} $\epsilon$-transducers as the composition of other $\epsilon$-transducers.  

Crucially, transducer factorization provides a principled route to modular causal-state decomposition: it breaks a monolithic belief state into interpretable sub-belief states, each aligned with a particular informational subsystem. This links structural decomposability with semantic interpretability.  The causal states of a machine reflect belief states \citep{rosasai}.  Rather than inferring intractably complex belief states from high-dimensional data in a large environment, we can leverage its natural modularity to localize causal structure within subcomponents, substantially accelerating computation.  This creates a bridge between structural transparency and computational efficiency, with direct implications for AI safety and interpretability.  By making inference tractable and modular, factored transducers open the door to scalable, mechanistically grounded world models.

\section{Discussion}

We have developed a framework for composing, decomposing, and coarse-graining transducers that treats world models as modular input–output machines rather than opaque monoliths. 
Starting from the notion of interfaces and their realization as transducers, we showed how layered structures of stochastic dynamics can be represented algebraically and combined into larger feedforward networks. 
Building on a long line of work on transducer and automata composition~\citep{schutzenberger1961remark,hopcroft1979introduction,mohri-1997-finite,mohri2002weighted}, our composition is formulated both in terms of stochastic kernels and their associated linear operators, and can be used to ensure that complex world models formed by stacking and wiring simpler components are still describable as a single higher-level transducer. This gives a principled way to move between `micro' and `macro' views of an environment~\citep{rosas2024software}, treating networks of mechanisms and their aggregate behaviour within one unified formalism.

The central contribution of this work is to highlight conditions under which a large transducer can be factored into a network of simpler sub-transducers, following similar approaches specialised in deterministic automata~\citep{krohn-rhodes-1965-algebraic,eilenberg1976automata,maler2010krohnrhodes,egri-nagy2005algebraic}. 
That said, it is important to note that our results do not factor a single process into prime processes --- as e.g. the Krohn-Rhodes decomposition~\citep{krohn-rhodes-1965-algebraic} does for semigroups (see \citet{maler1993decomposition} for probabilistic extensions). 
Instead, our results decompose multiple processes according to their causal dependencies, being closer to the causal graph discovery literature~\citep{glymour2019review,nogueira2022methods}.
Using novel information-theoretic quantities --- Intransducibility when latents are available, and Acausality when they are not --- we provide tests for when a subset of variables can be treated as the output of a causal module driven by the rest. 
These conditions lead to decomposition algorithms that `peel off' modules one by one, yielding a directed acyclic graph that captures a causal architecture implicit in the joint process.

Using these methods, we show how a seemingly entangled world model can --- when it obeys appropriate conditional independence constraints --- be rewritten as a composition of prime transducers, each associated with a distinct functional subsystem. 
This structural decomposition is not only conceptually clarifying but also computationally useful. Once a process has been factored into an ordered set of modules, learning and inference can proceed locally within each sub-transducer conditioned on the upstream modules that influence it. Instead of inferring a single huge latent state space, we can learn smaller modules, each with its own belief dynamics, and then compose them.
This creates natural opportunities for parallelizable inference and training~\citep{guestrin2003efficient} --- provided that the factorization captures genuine conditional independence rather than merely approximate separation.

We also showed how these ideas interact with coarse-graining and minimal predictive models. Sequences of causally adjacent transducers can be grouped into larger effective units without violating the feedforward structure, enabling multiscale descriptions that range from fine-grained micro-dynamics to high-level macro-modules. 
When each component is represented by its $\epsilon$-transducer, we proved that composition preserves minimality: the product of minimal predictors is itself the minimal predictor of the composed interface. 
This result connects our structural factorization to the emerging literature on belief states and geometric causal structure in neural networks, suggesting that the belief states of AI systems may themselves be decomposable into interacting belief subspaces, rather than taking place in a monolithic high-dimensional latent space.

There are, however, several important limitations. First, computing Intransducibility or Acausality requires access to joint distributions over long histories~\citep{jurgens2021divergent,marzen2016predictive}, which can make it intractable. In practice one may only need finite-horizon, low-order, or variational approximations, together with statistical tests that distinguish genuine structure from sampling noise~\citep{fiderer2025work,rosasai}. 
Second, we have restricted attention to feedforward, mechanistically stationary transducers, whereas many realistic agent–environment systems involve feedback, non-stationarity, and adaptation~\citep{beer1995dynamical,zhu2022adaptive}. Extending our results to such settings --- while preserving a useful notion of modularity --- remains an open challenge. 
Third, we have focused on formal properties and did not include empirical case studies; validating these ideas on trained world models, reinforcement learning environments, and large neural networks is a priority for future work.

Despite these caveats, we believe the framework developed here offers a promising route toward structurally transparent and computationally efficient world models. 
By treating environments and agents as composable, decomposable transducers, we gain a language in which to design, analyse, and interrogate complex AI systems. 
For safety and alignment, this provides a way to link mechanistic interpretability --- through causal states and belief dynamics --- to the modular organization of the broader agent--environment loop. 
Ultimately, our hope is that tools of this kind will help support AI systems whose internal structure is not only powerful and data-efficient, but also intelligible, auditable, and amenable to principled control.

%%%%%%%%%%%%%%%%%%%%%%%%%%%%%%%%%%%%%%%%%%%%%%%%%%%%%%%%%%%%%%%%
%% Appendices
%%%%%%%%%%%%%%%%%%%%%%%%%%%%%%%%%%%%%%%%%%%%%%%%%%%%%%%%%%%%%%%%
\appendix

\section*{Acknowledgments}
\label{sec:ack}
The authors thank Artemy Kolchinsky, Daniel Polani, Paul Riechers, Adam Shai, and Lucas Teixeira for insightful discussions that helped to shape this work. 
The work of A.B. and F.R. was supported by the UK ARIA Safeguarded AI programme.
The work of F.R. was also supported by the PIBBSS Affiliateship programme and by Open Philanthropy. 
The work of D.H. was supported by a UKRI AI World Leading Researcher Fellowship (grant number EP/W002949/1) awarded to Michael Wooldridge.
M.B. was supported by JST, Moonshot R\&D, Grant Number JPMJMS2012.

%%%%%%%%%%%%%%%%%%%%%%%%%%%%%%%%%%%%%%%%%%%%%%%%%%%%%%%%%%%%%%%%
%% NOTE: THIS MARKS THE END OF THE "MAIN TEXT"
%%%%%%%%%%%%%%%%%%%%%%%%%%%%%%%%%%%%%%%%%%%%%%%%%%%%%%%%%%%%%%%%

%%%%%%%%%%%%%%%%%%%%%%%%%%%%%%%%%%%%%%%%%%%%%%%%%%%%%%%%%%%%%%%%
%% Bibliography
%%%%%%%%%%%%%%%%%%%%%%%%%%%%%%%%%%%%%%%%%%%%%%%%%%%%%%%%%%%%%%%%
\bibliography{main}

@article{marzen2017nearly,
  title={Nearly maximally predictive features and their dimensions},
  author={Marzen, Sarah E and Crutchfield, James P},
  journal={Physical Review E},
  volume={95},
  number={5},
  pages={051301},
  year={2017},
  publisher={APS}
}

@article{hallak2015contextual,
  title={Contextual markov decision processes},
  author={Hallak, Assaf and Di Castro, Dotan and Mannor, Shie},
  journal={arXiv preprint arXiv:1502.02259},
  year={2015}
}

@article{boutilier2018planning,
  title={Planning and learning with stochastic action sets},
  author={Boutilier, Craig and Cohen, Alon and Daniely, Amit and Hassidim, Avinatan and Mansour, Yishay and Meshi, Ofer and Mladenov, Martin and Schuurmans, Dale},
  journal={arXiv preprint arXiv:1805.02363},
  year={2018}
}

@book{altman2021constrained,
  title={Constrained Markov decision processes},
  author={Altman, Eitan},
  year={2021},
  publisher={Routledge}
}

@article{bernstein2002complexity,
  title={The complexity of decentralized control of Markov decision processes},
  author={Bernstein, Daniel S and Givan, Robert and Immerman, Neil and Zilberstein, Shlomo},
  journal={Mathematics of operations research},
  volume={27},
  number={4},
  pages={819--840},
  year={2002},
  publisher={INFORMS}
}

@phdthesis{poupart2005exploiting,
    title={Exploiting structure to efficiently solve large scale partially observable Markov decision processes},
    author={Poupart, Pascal},
    year={2005},
    school={University of Toronto Toronto, Canada}
}

@inproceedings{icarte2018using,
  title={Using reward machines for high-level task specification and decomposition in reinforcement learning},
  author={Icarte, Rodrigo Toro and Klassen, Toryn and Valenzano, Richard and McIlraith, Sheila},
  booktitle={International Conference on Machine Learning},
  pages={2107--2116},
  year={2018},
  organization={PMLR}
}

@article{shai2026transformers,
  title={Transformers learn factored representations},
  author={Shai, Adam and Amdahl-Culleton, Loren and Christensen, Casper L and Bigelow, Henry R and Rosas, Fernando E and Boyd, Alexander B and Alt, Eric A and Ray, Kyle J and Riechers, Paul M},
  journal={arXiv preprint arXiv:2602.02385},
  year={2026}
}

@article{boutilier2000stochastic,
  title={Stochastic dynamic programming with factored representations},
  author={Boutilier, Craig and Dearden, Richard and Goldszmidt, Mois{\'e}s},
  journal={Artificial intelligence},
  volume={121},
  number={1-2},
  pages={49--107},
  year={2000},
  publisher={Elsevier}
}

@article{rosas2024software,
  title={Software in the natural world: A computational approach to hierarchical emergence},
  author={Rosas, Fernando E and Geiger, Bernhard C and Luppi, Andrea I and Seth, Anil K and Polani, Daniel and Gastpar, Michael and Mediano, Pedro AM},
  journal={arXiv preprint arXiv:2402.09090},
  year={2024}
}

@article{tafazoli2025building,
  title={Building compositional tasks with shared neural subspaces},
  author={Tafazoli, Sina and Bouchacourt, Flora M and Ardalan, Adel and Markov, Nikola T and Uchimura, Motoaki and Mattar, Marcelo G and Daw, Nathaniel D and Buschman, Timothy J},
  journal={Nature},
  year={2025},
  doi={10.1038/s41586-025-09805-2},
  isbn={1476-4687}
}

@article{martinez2024decomposing,
  title={Decomposing causality into its synergistic, unique, and redundant components},
  author={Mart{\'\i}nez-S{\'a}nchez, {\'A}lvaro and Arranz, Gonzalo and Lozano-Dur{\'a}n, Adri{\'a}n},
  journal={Nature Communications},
  volume={15},
  number={1},
  pages={9296},
  year={2024},
  publisher={Nature Publishing Group UK London}
}

@article{runge2019detecting,
  title={Detecting and quantifying causal associations in large nonlinear time series datasets},
  author={Runge, Jakob and Nowack, Peer and Kretschmer, Marlene and Flaxman, Seth and Sejdinovic, Dino},
  journal={Science advances},
  volume={5},
  number={11},
  pages={eaau4996},
  year={2019},
  publisher={American Association for the Advancement of Science}
}

@article{mandal2012work,
  title={Work and information processing in a solvable model of Maxwell’s demon},
  author={Mandal, Dibyendu and Jarzynski, Christopher},
  journal={Proceedings of the National Academy of Sciences},
  volume={109},
  number={29},
  pages={11641--11645},
  year={2012},
  publisher={National Academy of Sciences}
}

@inproceedings{tang2024prioritizing,
  title={Prioritizing safeguarding over autonomy: Risks of llm agents for science},
  author={Tang, Xiangru and Jin, Qiao and Zhu, Kunlun and Yuan, Tongxin and Zhang, Yichi and Zhou, Wangchunshu and Qu, Meng and Zhao, Yilun and Tang, Jian and Zhang, Zhuosheng and others},
  booktitle={ICLR 2024 Workshop on Large Language Model (LLM) Agents},
  year={2024}
}

@article{bengio2024international,
  title={International scientific report on the safety of advanced ai (interim report)},
  author={Bengio, Yoshua and Mindermann, S{\"o}ren and Privitera, Daniel and Besiroglu, Tamay and Bommasani, Rishi and Casper, Stephen and Choi, Yejin and Goldfarb, Danielle and Heidari, Hoda and Khalatbari, Leila and others},
  journal={arXiv preprint arXiv:2412.05282},
  year={2024}
}

@article{glanois2024survey,
  title={A survey on interpretable reinforcement learning},
  author={Glanois, Claire and Weng, Paul and Zimmer, Matthieu and Li, Dong and Yang, Tianpei and Hao, Jianye and Liu, Wulong},
  journal={Machine Learning},
  volume={113},
  number={8},
  pages={5847--5890},
  year={2024},
  publisher={Springer}
}

@article{rajeswaran2017learning,
  title={Learning complex dexterous manipulation with deep reinforcement learning and demonstrations},
  author={Rajeswaran, Aravind and Kumar, Vikash and Gupta, Abhishek and Vezzani, Giulia and Schulman, John and Todorov, Emanuel and Levine, Sergey},
  journal={arXiv preprint arXiv:1709.10087},
  year={2017}
}

@article{ouyang2022training,
  title={Training language models to follow instructions with human feedback},
  author={Ouyang, Long and Wu, Jeffrey and Jiang, Xu and Almeida, Diogo and Wainwright, Carroll and Mishkin, Pamela and Zhang, Chong and Agarwal, Sandhini and Slama, Katarina and Ray, Alex and others},
  journal={Advances in neural information processing systems},
  volume={35},
  pages={27730--27744},
  year={2022}
}

@article{andrychowicz2020learning,
  title={Learning dexterous in-hand manipulation},
  author={Andrychowicz, OpenAI: Marcin and Baker, Bowen and Chociej, Maciek and Jozefowicz, Rafal and McGrew, Bob and Pachocki, Jakub and Petron, Arthur and Plappert, Matthias and Powell, Glenn and Ray, Alex and others},
  journal={The International Journal of Robotics Research},
  volume={39},
  number={1},
  pages={3--20},
  year={2020},
  publisher={SAGE Publications Sage UK: London, England}
}

@article{berner2019dota,
  title={Dota 2 with large scale deep reinforcement learning},
  author={Berner, Christopher and Brockman, Greg and Chan, Brooke and Cheung, Vicki and D{\k{e}}biak, Przemys{\l}aw and Dennison, Christy and Farhi, David and Fischer, Quirin and Hashme, Shariq and Hesse, Chris and others},
  journal={arXiv preprint arXiv:1912.06680},
  year={2019}
}

@article{mohri2002weighted,
  title={Weighted finite-state transducers in speech recognition},
  author={Mohri, Mehryar and Pereira, Fernando and Riley, Michael},
  journal={Computer Speech \& Language},
  volume={16},
  number={1},
  pages={69--88},
  year={2002},
  publisher={Elsevier}
}

@article{jurgens2021shannon,
  title={Shannon entropy rate of hidden Markov processes},
  author={Jurgens, Alexandra M and Crutchfield, James P},
  journal={Journal of Statistical Physics},
  volume={183},
  number={2},
  pages={32},
  year={2021},
  publisher={Springer}
}

@article{boyd2016identifying,
  title={Identifying functional thermodynamics in autonomous Maxwellian ratchets},
  author={Boyd, Alexander B and Mandal, Dibyendu and Crutchfield, James P},
  journal={New Journal of Physics},
  volume={18},
  number={2},
  pages={023049},
  year={2016},
  publisher={IOP Publishing}
}

@inproceedings{rosasai,
  title={AI in a vat: Fundamental limits of efficient world modelling for safe agent sandboxing},
  author={Rosas, Fernando and Boyd, Alexander and Baltieri, Manuel},
  booktitle={Reinforcement Learning Conference},
  year={2025}
}

@article{fiderer2025work,
  title={The Work Capacity of Channels with Memory: Maximum Extractable Work in Percept-Action Loops},
  author={Fiderer, Lukas J and Barth, Paul C and Smith, Isaac D and Briegel, Hans J},
  journal={arXiv preprint arXiv:2504.06209},
  year={2025}
}

@article{marzen2016predictive,
  title={Predictive rate-distortion for infinite-order Markov processes},
  author={Marzen, Sarah E and Crutchfield, James P},
  journal={Journal of Statistical Physics},
  volume={163},
  pages={1312--1338},
  year={2016},
  publisher={Springer}
}

@article{jurgens2021divergent,
  title={Divergent predictive states: The statistical complexity dimension of stationary, ergodic hidden Markov processes},
  author={Jurgens, Alexandra M and Crutchfield, James P},
  journal={Chaos: An Interdisciplinary Journal of Nonlinear Science},
  volume={31},
  number={8},
  year={2021},
  publisher={AIP Publishing}
}

@article{ellison2011information,
  title={Information symmetries in irreversible processes},
  author={Ellison, Christopher J and Mahoney, John R and James, Ryan G and Crutchfield, James P and Reichardt, J{\"o}rg},
  journal={Chaos: An Interdisciplinary Journal of Nonlinear Science},
  volume={21},
  number={3},
  year={2011},
  publisher={AIP Publishing}
}

@article{piotrowski2025constrained,
  title={Constrained belief updates explain geometric structures in transformer representations},
  author={Piotrowski, Mateusz and Riechers, Paul M and Filan, Daniel and Shai, Adam S},
  journal={arXiv preprint arXiv:2502.01954},
  year={2025}
}

@article{boyd2024thermodynamic,
  title={Thermodynamic Overfitting and Generalization: {E}nergetic Limits on Predictive Complexity},
  author={Boyd, Alexander B and Crutchfield, James P and Gu, Mile and Binder, Felix C},
  journal={arXiv preprint arXiv:2402.16995},
  year={2024}
}

@article{boyd2018thermodynamics,
  title={Thermodynamics of modularity: {S}tructural costs beyond the Landauer bound},
  author={Boyd, Alexander and Mandal, Dibyendu and Crutchfield, James},
  journal={Physical Review X},
  volume={8},
  number={3},
  pages={031036},
  year={2018},
  publisher={APS}
}

@article{crutchfield2012between,
  title={Between order and chaos},
  author={Crutchfield, James},
  journal={Nature Physics},
  volume={8},
  number={1},
  pages={17--24},
  year={2012},
  publisher={Nature Publishing Group UK London}
}

@article{crutchfield1989inferring,
  title={Inferring statistical complexity},
  author={Crutchfield, James and Young, Karl},
  journal={Physical review letters},
  volume={63},
  number={2},
  pages={105},
  year={1989},
  publisher={APS}
}

@book{sutton1998introduction,
    title={Reinforcement Learning: {A}n Introduction},
    author={Sutton, Richard S. and Barto, Andrew G.},
    publisher={The MIT Press},
    year={1998},
    address={Cambridge, MA},
}

@article{kaelbling1998planning,
  title={Planning and acting in partially observable stochastic domains},
  author={Kaelbling, Leslie Pack and Littman, Michael and Cassandra, Anthony},
  journal={Artificial intelligence},
  volume={101},
  number={1-2},
  pages={99--134},
  year={1998},
  publisher={Elsevier}
}

@article{barnett2015computational,
  title={Computational mechanics of input--output processes: {S}tructured transformations and the $\epsilon$-transducer},
  author={Barnett, Nix and Crutchfield, James},
  journal={Journal of Statistical Physics},
  volume={161},
  number={2},
  pages={404--451},
  year={2015},
  publisher={Springer}
}

@article{dalrymple2024towards,
  title={Towards Guaranteed Safe {AI}: {A} Framework for Ensuring Robust and Reliable AI Systems},
  author={Dalrymple, David and Skalse, Joar and Bengio, Yoshua and Russell, Stuart and Tegmark, Max and Seshia, Sanjit and Omohundro, Steve and Szegedy, Christian and Goldhaber, Ben and Ammann, Nora and others},
  journal={arXiv preprint arXiv:2405.06624},
  year={2024}
}

@article{shai2025transformers,
  title={Transformers represent belief state geometry in their residual stream},
  author={Shai, Adam and Teixeira, Lucas and Oldenziel, Alexander and Marzen, Sarah and Riechers, Paul},
  journal={Advances in Neural Information Processing Systems},
  volume={37},
  pages={75012--75034},
  year={2025}
}

@article{shalizi2001computational,
  title={Computational mechanics: {P}attern and prediction, structure and simplicity},
  author={Shalizi, Cosma and Crutchfield, James},
  journal={Journal of Statistical Physics},
  volume={104},
  number={3-4},
  pages={817--879},
  year={2001},
  publisher={Springer}
}

@article{diaz2023connecting,
  title={Connecting the dots in trustworthy Artificial Intelligence: From {AI} principles, ethics, and key requirements to responsible {AI} systems and regulation},
  author={D{\'\i}az-Rodr{\'\i}guez, Natalia and Del Ser, Javier and Coeckelbergh, Mark and de Prado, Marcos L{\'o}pez and Herrera-Viedma, Enrique and Herrera, Francisco},
  journal={Information Fusion},
  volume={99},
  pages={101896},
  year={2023},
  publisher={Elsevier}
}

@INPROCEEDINGS{hafner2019dream,
  title={Dream to control: Learning behaviors by latent imagination},
  author={Hafner, Danijar and Lillicrap, Timothy and Ba, Jimmy and Norouzi, Mohammad},
  booktitle={Proceedings of the International Conference on Learning Representations (ICLR’20)},
  year={2020}
}

@article{baek2025dreamweaver,
  title={Dreamweaver: Learning Compositional World Representations from Pixels},
  author={Baek, Junyeob and Wu, Yi-Fu and Singh, Gautam and Ahn, Sungjin},
  journal={arXiv preprint arXiv:2501.14174},
  year={2025}
}

@INPROCEEDINGS{hansen2023td,
  title={{TD-MPC2}: {S}calable, robust world models for continuous control},
  author={Hansen, Nicklas and Su, Hao and Wang, Xiaolong},
  booktitle={Proceedings of the International Conference on Learning Representations (ICLR’24)},
  year={2024}
}

@article{zhu2022adaptive,
  title={Adaptive deep reinforcement learning for non-stationary environments},
  author={Zhu, Jin and Wei, Yutong and Kang, Yu and Jiang, Xiaofeng and Dullerud, Geir E},
  journal={Science China Information Sciences},
  volume={65},
  number={10},
  pages={202204},
  year={2022},
  publisher={Springer}
}

@article{beer1995dynamical,
  title={A dynamical systems perspective on agent-environment interaction},
  author={Beer, Randall D},
  journal={Artificial intelligence},
  volume={72},
  number={1-2},
  pages={173--215},
  year={1995},
  publisher={Elsevier}
}

@article{guestrin2003efficient,
  title={Efficient solution algorithms for factored MDPs},
  author={Guestrin, Carlos and Koller, Daphne and Parr, Ronald and Venkataraman, Shobha},
  journal={Journal of Artificial Intelligence Research},
  volume={19},
  pages={399--468},
  year={2003}
}

@book{schutzenberger1961remark,
  title={A remark on finite transducers},
  author={Sch{\"u}tzenberger, Marcel P},
  year={1961},
  publisher={Mathematical Sciences Directorate, Air Force Office of Scientific Research}
}

@inproceedings{bruce2024genie,
  title={Genie: Generative interactive environments},
  author={Bruce, Jake and Dennis, Michael D and Edwards, Ashley and Parker-Holder, Jack and Shi, Yuge and Hughes, Edward and Lai, Matthew and Mavalankar, Aditi and Steigerwald, Richie and Apps, Chris and others},
  booktitle={Forty-first International Conference on Machine Learning},
  year={2024}
}

@article{ding2025understanding,
  title={Understanding world or predicting future? a comprehensive survey of world models},
  author={Ding, Jingtao and Zhang, Yunke and Shang, Yu and Zhang, Yuheng and Zong, Zefang and Feng, Jie and Yuan, Yuan and Su, Hongyuan and Li, Nian and Sukiennik, Nicholas and others},
  journal={ACM Computing Surveys},
  volume={58},
  number={3},
  pages={1--38},
  year={2025},
  publisher={ACM New York, NY}
}

@article{littman2001predictive,
  title={Predictive representations of state},
  author={Littman, Michael and Sutton, Richard S},
  journal={Advances in neural information processing systems},
  volume={14},
  year={2001}
}

@inproceedings{reape-thompson-1988-parallel,
    title = "Parallel Intersection and Serial Composition of Finite State Transducers",
    author = "Reape, Mike  and
      Thompson, Henry",
    booktitle = "{C}oling {B}udapest 1988 Volume 2: {I}nternational {C}onference on {C}omputational {L}inguistics",
    year = "1988",
}

@book{hopcroft1979introduction,
  title={Introduction to Automata Theory, Languages, and Computation},
  author={Hopcroft, John and Ullman, Jeffrey},
  year={1979},
  publisher={Addison-Wesley Publishing Company, Reading, MA}
}

@inbook{maler2010krohnrhodes,
author = {Maler, Oded},
title = {On the Krohn-Rhodes cascaded decomposition theorem},
year = {2010},
isbn = {3642137539},
publisher = {Springer-Verlag},
address = {Berlin, Heidelberg},
booktitle = {Time for Verification: Essays in Memory of Amir Pnueli},
pages = {260–278},
numpages = {19}
}

@article{krohn-rhodes-1965-algebraic,
 ISSN = {00029947, 10886850},
 author = {Kenneth Krohn and John Rhodes},
 journal = {Transactions of the American Mathematical Society},
 pages = {450--464},
 publisher = {American Mathematical Society},
 title = {Algebraic Theory of Machines. I. Prime Decomposition Theorem for Finite Semigroups and Machines},
 urldate = {2025-04-24},
 volume = {116},
 year = {1965}
}

@book{eilenberg1976automata,
  title={Automata, Languages, and Machines},
  author={Eilenberg, S. and Tilson, B.},
  isbn={9780080873756},
  series={Pure and Applied Mathematics},
  year={1976},
  publisher={Academic Press}
}

@article{glymour2019review,
  title={Review of causal discovery methods based on graphical models},
  author={Glymour, Clark and Zhang, Kun and Spirtes, Peter},
  journal={Frontiers in genetics},
  volume={10},
  pages={524},
  year={2019},
  publisher={Frontiers Media SA}
}

@book{kemeny1969finite,
  title={Finite markov chains},
  author={Kemeny, John G and Snell, J Laurie and others},
  volume={26},
  year={1969},
  publisher={van Nostrand Princeton, NJ}
}

@article{nogueira2022methods,
  title={Methods and tools for causal discovery and causal inference},
  author={Nogueira, Ana Rita and Pugnana, Andrea and Ruggieri, Salvatore and Pedreschi, Dino and Gama, Jo{\~a}o},
  journal={Wiley interdisciplinary reviews: data mining and knowledge discovery},
  volume={12},
  number={2},
  pages={e1449},
  year={2022},
  publisher={Wiley Online Library}
}

@inproceedings{egri-nagy2005algebraic,
	address = {Berlin, Heidelberg},
	author = {Egri-Nagy, Attila and Nehaniv, Chrystopher L.},
	booktitle = {Implementation and Application of Automata},
	editor = {Domaratzki, Michael and Okhotin, Alexander and Salomaa, Kai and Yu, Sheng},
	isbn = {978-3-540-30500-2},
	pages = {315--316},
	publisher = {Springer Berlin Heidelberg},
	title = {Algebraic Hierarchical Decomposition of Finite State Automata: Comparison of Implementations for Krohn-Rhodes Theory},
	year = {2005}
}

@article{jansma2025decomposing,
  title={Decomposing interventional causality into synergistic, redundant, and unique components},
  author={Jansma, Abel},
  journal={arXiv preprint arXiv:2501.11447},
  year={2025}
}

@article{mohri-1997-finite,
    title = "Finite-State Transducers in Language and Speech Processing",
    author = "Mohri, Mehryar",
    editor = "Hirschberg, Julia",
    journal = "Computational Linguistics",
    volume = "23",
    number = "2",
    year = "1997",
    address = "Cambridge, MA",
    publisher = "MIT Press",
    pages = "269--311"
}

@article{SCHUTZENBERGER1961185,
	author = {M.P. Sch{\"u}tzenberger},
	doi = {https://doi.org/10.1016/S0019-9958(61)80006-5},
	issn = {0019-9958},
	journal = {Information and Control},
	number = {2},
	pages = {185-196},
	title = {A remark on finite transducers},
	volume = {4},
	year = {1961}}

@InProceedings{mosbach2025sold,
  title = 	 {{SOLD}: Slot Object-Centric Latent Dynamics Models for Relational Manipulation Learning from Pixels},
  author =       {Mosbach, Malte and Ewertz, Jan Niklas and Villar-Corrales, Angel and Behnke, Sven},
  booktitle = 	 {Proceedings of the 42nd International Conference on Machine Learning},
  pages = 	 {44911--44935},
  year = 	 {2025},
  volume = 	 {267},
}

@inproceedings{
goyal2021recurrent,
title={Recurrent Independent Mechanisms},
author={Anirudh Goyal and Alex Lamb and Jordan Hoffmann and Shagun Sodhani and Sergey Levine and Yoshua Bengio and Bernhard Sch{\"o}lkopf},
booktitle={International Conference on Learning Representations},
year={2021},
}

@misc{rodriguezsanchez2025pixelsfactorslearningindependently,
      title={From Pixels to Factors: Learning Independently Controllable State Variables for Reinforcement Learning}, 
      author={Rafael Rodriguez-Sanchez and Cameron Allen and George Konidaris},
      year={2025},
      eprint={2510.02484},
      archivePrefix={arXiv},
      primaryClass={cs.LG},
}

@article{liu2023learning,
  title={Learning world models with identifiable factorization},
  author={Liu, Yuren and Huang, Biwei and Zhu, Zhengmao and Tian, Honglong and Gong, Mingming and Yu, Yang and Zhang, Kun},
  journal={Advances in Neural Information Processing Systems},
  volume={36},
  pages={31831--31864},
  year={2023}
}

@misc{hafner2025trainingagentsinsidescalable,
      title={Training Agents Inside of Scalable World Models}, 
      author={Danijar Hafner and Wilson Yan and Timothy Lillicrap},
      year={2025},
      eprint={2509.24527},
      archivePrefix={arXiv},
      primaryClass={cs.AI},
}

@inproceedings{zhao2022toward,
  title={Toward compositional generalization in object-oriented world modeling},
  author={Zhao, Linfeng and Kong, Lingzhi and Walters, Robin and Wong, Lawson LS},
  booktitle={International Conference on Machine Learning},
  pages={26841--26864},
  year={2022},
  organization={PMLR}
}

@article{feng2023learning,
  title={Learning dynamic attribute-factored world models for efficient multi-object reinforcement learning},
  author={Feng, Fan and Magliacane, Sara},
  journal={Advances in Neural Information Processing Systems},
  volume={36},
  pages={19117--19144},
  year={2023}
}

@article{carlsson2015prime,
  title={A Prime Decomposition of Probabilistic Automata},
  author={Carlsson, Gunnar and Yu, Jun},
  journal={arXiv preprint arXiv:1503.01502},
  year={2015}
}

@inproceedings{maler1993decomposition,
  title={A decomposition theorem for probabilistic transition systems},
  author={Maler, Oded},
  booktitle={Annual Symposium on Theoretical Aspects of Computer Science},
  pages={323--332},
  year={1993},
  organization={Springer}
}

@article{dohmen2022inferring, title={Inferring Probabilistic Reward Machines from Non-Markovian Reward Signals for Reinforcement Learning}, volume={32}, url={https://ojs.aaai.org/index.php/ICAPS/article/view/19844}, DOI={10.1609/icaps.v32i1.19844}, abstractNote={The success of reinforcement learning in typical settings is predicated on Markovian assumptions on the reward signal by which an agent learns optimal policies. In recent years, the use of reward machines has relaxed this assumption by enabling a structured representation of non-Markovian rewards. In particular, such representations can be used to augment the state space of the underlying decision process, thereby facilitating non-Markovian reinforcement learning. However, these reward machines cannot capture the semantics of stochastic reward signals. In this paper, we make progress on this front by introducing probabilistic reward machines (PRMs) as a representation of non-Markovian stochastic rewards. We present an algorithm to learn PRMs from the underlying decision process and prove results around its correctness and convergence.}, number={1}, journal={Proceedings of the International Conference on Automated Planning and Scheduling}, author={Dohmen, Taylor and Topper, Noah and Atia, George and Beckus, Andre and Trivedi, Ashutosh and Velasquez, Alvaro}, year={2022}, month={Jun.}, pages={574-582} }
\bibliographystyle{rlj}

\beginSupplementaryMaterials

\section{Feedback interfaces}
\label{app:Interfaces}

In the context of a perception-action loop linking an agent and an environment, the environment can be thought of as a system that stochastically turns action sequences $x_{0:t}=x_0 \cdots x_{t-1}$ into observation sequences $y_{0:t}=y_0 \cdots y_{t-1}$ for any time length $t$.  Similarly, the agent operates analogously, stochastically transforming observation sequences into action sequences.  The action sequence is the agent's output and the environment's input, while the observation sequence is the agent's input and the environment's output.  Both agent and environment are specified by an interface $\mathcal{I}$~\citep{rosasai}, which produces an output process for each input sequence~\citep{fiderer2025work}
\begin{align}
    \mathcal{I}_E(y_{0:t}|x_{0:t}) &=  \Pr(Y_{0:t}=y_{0:t}|X_{0:t}=x_{0:t}) \nonumber \\
    \mathcal{I}_A(x_{0:t}|y_{0:t}) &=  \Pr(X_{0:t}=x_{0:t}|Y_{0:t}=y_{0:t}).
\end{align}
Here, the capital variables represent random variables while the lowercase represent specific realizations.  Together, they produce the joint probability of an action-observation sequence in the perception-action loop~\citep{fiderer2025work}:
\begin{align}
\Pr(X_{0:t}=x_{0:t},Y_{0:t}=y_{0:t})=\mathcal{I}_E(y_{0:t}|x_{0:t})\mathcal{I}_A(x_{0:t}|y_{0:t}).
\end{align}
While it may appear that this equality does not obey Bayes' rule, it can be derived using the causal architecture of an interface and the fact that it can be described via a transducer.  Let us say that the environment is given by a kernel $e(y,r'|x,r)=\Pr(Y_t=y,R_{t+1}=r'|X_t=x,R_t=r)$, and the kernel of the agent is given by $a(x,s'|y,s)=\Pr(X_{t+1}=x,S_{t+1}=s'|Y_t=y,S_t=s)$, where $S_t$ is the latent of the agent and $R_t$ is the latent of the environment.  From this, we can determine the joint distribution of the perception-action loop through recursive operation
\begin{align}
    & \Pr(X_{0:t}=x_{0:t},Y_{0:t}=y_{0:t}) \nonumber
    \\  = &\sum_{s_{0:t},r_{0:t+1}}\Pr(X_{0:t}=x_{0:t},Y_{0:t}=y_{0:t},R_{0:t+1}=r_{0:t+1},S_{0:t+1}=s_{0:t+1}) \nonumber \\ 
    % = & \sum_{s_{0:t}, r_{0:t+1}} p(x_{0:t}, y_{0:t}, r_{0:t+1}, s_{0:t+1}) \nonumber \\ 
    % = & \sum_{s_{0:t}, r_{0:t+1}} p(x_{1:t}, y_{0:t}, r_{1:t+1}, s_{1:t+1} \mid x_{0}, r_{0}, s_{0}) p(x_{0}, r_{0}, s_{0}) \nonumber \\ 
    % = & \sum_{s_{0:t}, r_{0:t+1}} 
    % p(x_{2:t}, y_{1:t}, r_{2:t+1}, s_{2:t+1} \mid x_{1}, y_{0}, r_{1}, s_{1}) 
    % p(x_{1}, y_{0}, r_{1}, s_{1} \mid x_{0}, r_{0}, s_{0}) 
    % p(x_{0}, r_{0}, s_{0}) \nonumber \\ 
    % = & \sum_{s_{0:t}, r_{0:t+1}} 
    % p(x_{2:t}, y_{1:t}, r_{2:t+1}, s_{2:t+1} \mid x_{1}, y_{0}, r_{1}, s_{1}) 
    % p(y_{0}, r_{1} \mid x_{0}, r_{0}) 
    % p(x_{1}, s_{1} \mid s_{0}) 
    % p(x_{0}, r_{0}, s_{0}) \nonumber \\ 
    % = & \sum_{s_{0:t}, r_{0:t+1}} 
    % p(x_{3:t}, y_{2:t}, r_{3:t+1}, s_{3:t+1} \mid x_{2}, y_{1}, r_{2}, s_{2}) 
    % p(x_{2}, y_{1}, r_{2}, s_{2} \mid x_{1}, y_{0}, r_{1}, s_{1}) 
    % p(y_{0}, r_{1} \mid x_{0}, r_{0}) \nonumber \\ 
    % & p(x_{1}, s_{1} \mid s_{0}) 
    % p(x_{0}, r_{0}, s_{0}) \nonumber \\ 
    % = & \sum_{s_{0:t}, r_{0:t+1}} 
    % p(x_{3:t}, y_{2:t}, r_{3:t+1}, s_{3:t+1} \mid x_{2}, y_{1}, r_{2}, s_{2}) 
    % p(y_{1}, r_{2} \mid x_{1}, r_{1}) 
    % p(x_{2}, s_{2} \mid y_{0}, s_{1}) 
    % p(y_{0}, r_{1} \mid x_{0}, r_{0}) \nonumber \\ 
    % & p(x_{1}, s_{1} \mid s_{0}) 
    % p(x_{0}, r_{0}, s_{0}) \nonumber \\
    % = & \sum_{s_{0:t}, r_{0:t+1}} 
    % p(x_{3:t}, y_{2:t}, r_{3:t+1}, s_{3:t+1} \mid x_{2}, y_{1}, r_{2}, s_{2}) 
    % % p(y_{1}, r_{2} \mid x_{1}, r_{1}) 
    % p(x_{2}, s_{2} \mid y_{0}, s_{1}) 
    % \prod_{i=0}^{t-1} p(y_{i}, r_{i+1} \mid x_{i}, r_{i}) \nonumber \\ 
    % & p(x_{1}, s_{1} \mid s_{0}) 
    % p(x_{0}, r_{0}, s_{0}) \nonumber
    \\  =& \sum_{s_{0:t},r_{0:t+1}} \Pr(X_0=x_0,S_0=s_0)\Pr(R_0=r_0)e(y_{t-1},r_{t}|x_{t-1},r_{t-1})
     \\ & \times \prod_{i=0}^{t-2}a(x_{i+1},s_{i+1}|y_i,s_i)e(y_{i},r_{i+1}|x_i,r_i)
    \\  =& \left( \sum_{s_{0:t}} \Pr(X_0=x_0,S_0=s_0) \prod_{i=0}^{t-2}a(x_{i+1},s_{i+1}|y_i,s_i) \right)
    \\ & \times \left( \sum_{r_{0:t+1}} \Pr(R_0=r_0)
     \prod_{i=0}^{t-1}e(y_{i},r_{i+1}|x_i,r_i) \right)
     \\ & = \Pr(X_{0:t}=x_{0:t}|Y_{0:t}=y_{0:t})\Pr(Y_{0:t}=y_{0:t}|X_{0:t}=x_{0:t})
     \\ & = \mathcal{I}_\text{A}(x_{0:t}|y_{0:t})\mathcal{I}_\text{E}(y_{0:t}|x_{0:t}).
\end{align}

The interface characterizes the behavior of the agent or environment, independent of the details of their internal models or other latents.  \autoref{fig:Perception-Action_Loop} shows how the perception-action loop can be decomposed into distinct interfaces in this way.

\begin{figure*}[ht]
\centering
\includegraphics[width=.6\columnwidth]{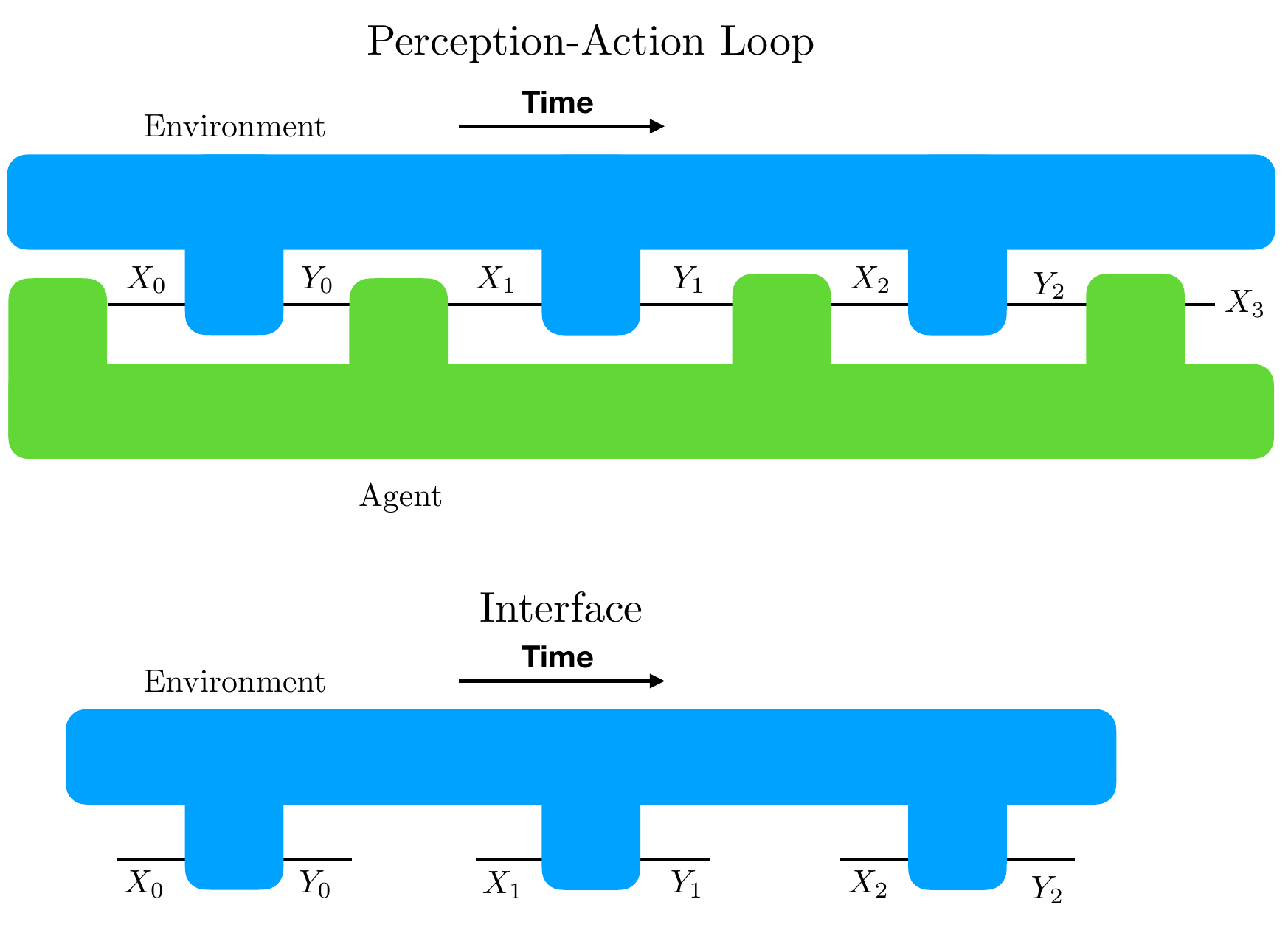}
\caption{A perception-action loop (top) operates through the exchange of actions $X_t$ that pass information from agent to environment, and observations $Y_t$ that pass information from the environment to the agent.  This interwoven circuit produces a joint process over actions and observations.  We can isolate the dynamics of either the agent or the environment, which are both represented by an interface.  This is a stochastic circuit from an input sequence to an output sequence.}
\label{fig:Perception-Action_Loop} 
\end{figure*}

Causal interfaces are well-suited to describe environments within perception-action loops, where $X_t$ corresponds to actions that precede observations $Y_t$. In a feedforward context, the conditional probabilities in the definition of the interface exactly match observed probabilities $\Pr(Y_{0:t}|X_{0:t})$. However, a perception-action loop involves \emph{feedback}, which can change the associated distributions by allowing past outputs $\overleftarrow{Y}_t$ to influence future inputs $\overrightarrow{X}_t$ \emph{indirectly} through a channel outside the interface in question.  Nevertheless, the interface allows us to exactly calculate observed probabilities of the feedback process through a product of the probabilities in the two interacting interfaces as shown above \citep{fiderer2025work}. 

In the feedforward case, the conditional probabilities that specify an interface are the same as observed conditional probabilities.  The independence of past outputs from future inputs is expressed in terms of distributions induced by a causal interface that satisfy  
\begin{align}
\Pr(\overleftarrow{Y}_t|\overleftarrow{X}_t,\overrightarrow{X}_t)=\Pr(\overleftarrow{Y}_t|\overleftarrow{X}_t).
\end{align} 
This probabilistic relation will not necessarily hold in a perception-action loop, because feedback from output pasts $\overleftarrow{Y}_t$ can indirectly influence the future inputs $\overrightarrow{X}_t$.  If the nonanticipatory condition is not satisfied in a feedforward network, then the interface somehow utilizes future inputs before receiving them, violating causality.  However, feedback networks can violate this condition, because the output past is processed and returned to the input future.

\section{Universality of mechanically stationary transducers}
\label{app:MechanicallyStationary}

\citet{rosasai} consider time-dependent kernels, such that the relationship between random variables shown in \autoref{eq:kernel} can change as a function of time.  However, let's show how mechanistically stationary transducers are just as expressive, capable of realizing any interface that their time-dependent counterparts can produce, given sufficient latent memory. 

Recall that \emph{any} nonanticipatory interface can be generated by a transducer.  Let us label the kernel of this transducer $T^{(y|x)}_{r \rightarrow r'}(t)$ with latent memory space $\mathcal{R}$.  We can construct a mechanically stationary transducer with kernel $T^{(y|x)}_{m \rightarrow m'}$ with a latent memory space $\mathcal{M}$ that is the direct product of the original latent memory space $\mathcal{R}$ and the space of time indices $\mathcal{T}=\{0,1,\cdots \}$.  We define the probabilities of this modified transducer via relation to the original kernel:
\begin{align}
    T^{(y|x)}_{r,t \rightarrow r',t'}=T^{(y|x)}_{r \rightarrow r'}(t) \delta_{t',t+1}.
\end{align}
Thus, the probability of following the input/output trajectory $y_{0:t}|x_{0:t}$ is the same
\begin{align}
    \mathcal{I}[y_{0:\tau}|x_{0:\tau}]& =\sum_{m_{0:\tau+1}}\Pr(M_0=m_0)\prod_{i=0}^{\tau-1}T^{(y_i|x_i)}_{m_i \rightarrow m_{i+1}}
    \\ & =\sum_{r_{0:\tau+1},t_{0:\tau+1}}\Pr(R_0=r_0,T_0=t_0)\prod_{i=0}^{\tau-1}T^{(y_i|x_i)}_{r_i,t_i \rightarrow r_{i+1},t_{i+1}}
    \\ & \text{assuming the initial time latent memory state is }t_0=0
    \\ & =\sum_{r_{0:\tau+1},t_{0:\tau+1}}\Pr(R_0=r_0)\delta_{t_0,0}\prod_{i=0}^{\tau-1}T^{(y_i|x_i)}_{r_i \rightarrow r_{i+1}}(t_i) \delta_{t_{i+1},t_{i}+1}
    \\ & \text{delta function selects out single trajectory } t_{0:\tau}=\{0, \cdots, \tau-1\}
     \\ & =\sum_{r_{0:\tau+1}}\Pr(R_0=r_0)\prod_{i=0}^{\tau-1}T^{(y_i|x_i)}_{r_i \rightarrow r_{i+1}}(i).
\end{align}
This final line is precisely the expression for the interface that one arrives at from the general transducer.  Therefore, the mechanically stationary transducer we design produces the same interface.

\section{Nonanticipatory transducers}
\label{app:Nonanticipatory Transducers}

\begin{proof}
We have our two properties:
\begin{enumerate}
    \item $\Pr(Y|X)$ is nonanticipatory, meaning that
    \begin{align}
\Pr(\overleftarrow{Y}_t|\overleftarrow{X}_t,X_{t:t+dt})=\Pr(\overleftarrow{Y}_t|\overleftarrow{X}_t)
    \end{align}
    \item $\Pr(Y|X)$ has a transducer presentation, meaning that there exists a $R$ process such that
    \begin{align}
\Pr(\overleftarrow{Y}_{t+1},\overleftarrow{R}_{t+2}|\overleftarrow{X}_{t+1})=\Pr(Y_t,R_{t+1}|X_t,R_t)\Pr(\overleftarrow{Y}_{t},\overleftarrow{R}_{t+1}|\overleftarrow{X}_{t})
    \end{align}
\end{enumerate}
\begin{itemize}
    \item $1. \Rightarrow 2. $: Let us decompose the probability of an output given an input
    \begin{align}
    \Pr(\overleftarrow{Y}_{t+1}|\overleftarrow{X}_{t+1})&=\Pr(\overleftarrow{Y}_tY_t|\overleftarrow{X}_tX_t)
\\&=\Pr(Y_t|\overleftarrow{Y}_t\overleftarrow{X}_tX_t)\Pr(\overleftarrow{Y}_t|\overleftarrow{X}_tX_t)
    \\ & \text{using the nonanticipatory condition}
    \\ & =\Pr(Y_t|\overleftarrow{Y}_t\overleftarrow{X}_tX_t)\Pr(\overleftarrow{Y}_t|\overleftarrow{X}_t).
    \end{align}
    Thus, we have a recursive relation for building up the interface time-step by time-step as a transducer. We can implement the transducer by simply choosing the latent memory to be a copy of the input-output history $R_t=\overleftarrow{X}_t\overleftarrow{Y}_t$.
    \begin{align}
\Pr(Y_t,R_{t+1}|X_t,R_t)\Pr(\overleftarrow{Y}_{t},\overleftarrow{R}_{t+1}|\overleftarrow{X}_{t}) & =\Pr(Y_t,\overleftarrow{X}_{t+1}\overleftarrow{Y}_{t+1}|X_t,\overleftarrow{X}_t\overleftarrow{Y}_t)\Pr(\overleftarrow{Y}_{t},\overleftarrow{X}_t\overleftarrow{Y}_t|\overleftarrow{X}_{t}) 
\\ & \text{eliminating redundancy using }\Pr(A,B,B|A)=\Pr(B|A) 
\\ & =\Pr(Y_t,|X_t,\overleftarrow{X}_t\overleftarrow{Y}_t)\Pr(\overleftarrow{Y}_{t}|\overleftarrow{X}_{t}) 
\\ & =\Pr(\overleftarrow{Y}_{t+1}|\overleftarrow{X}_{t+1}) 
\\ & \text{reintroducing redundancy}
\\ & = \Pr(\overleftarrow{Y}_{t+1},\overleftarrow{R}_{t+2}|\overleftarrow{X}_{t+1}).
    \end{align}
Thus, this satisfies the condition of being a transducer.
\item $2. \Rightarrow 1. :$ Let us decompose an expression for the output future and past, and the latent memory $R_t$ that links them, conditioned on the input past and history using the fact that the latent memory partitions past and future
\begin{align}
\Pr(\overleftarrow{Y}_t,\overrightarrow{Y}_t,R_t|\overleftarrow{X}_t,\overrightarrow{X}_t)& = \Pr(\overleftarrow{Y}_t|R_t,\overrightarrow{X}_t)\Pr(\overleftarrow{Y}_t,R_t|\overleftarrow{X}_t).
\end{align}
We simply sum to marginalize over the latent memory state $R_t$ and future outputs $\overrightarrow{Y}_t$ to obtain our desired expression for being nonanticipatory
\begin{align}
\Pr(\overleftarrow{Y}_t|\overleftarrow{X}_t,\overrightarrow{X}_t)& = \Pr(\overleftarrow{Y}_t|\overleftarrow{X}_t).
\end{align}
\end{itemize}
\end{proof}

\section{Proof of \autoref{res:kronecker}}
\label{app:kronecker}

\begin{proof}
The kernel of the composite transducer is the product of the kernels of each of its components
\begin{align}
    V^{(yz|x)}_{rs \rightarrow r's'} & = \Pr(Z_t=z,Y_t=y,R_{t+1}=r',S_{t+1}=s'|X_t=x,R_t=r,S_t=s)
    \\ & =\Pr(Z_t=z,S_{t+1}=s'|X_t=x,Y_t=y,S_t=s)\Pr(Y_t=y,R_{t+1}=r'|X_t=x,R_t=r)
    \\ & = U^{(z|xy)}_{s \rightarrow s'}T^{(y|x)}_{r \rightarrow r'}.
\end{align}
The latent space of the composite transducer $V$ is the direct product $\mathcal{R} \times \mathcal{S}$ of the latent spaces of its components $T$ and $U$.  The elementary vector $\mathbf{e}_{rs}$ in this vector space is the Kronecker product of the elementary vectors in the component vector spaces
\begin{align}
    \mathbf{e}_{rs}=\mathbf{e}_{r} \otimes \mathbf{e}_{s}.
\end{align}
Thus, the composite linear operator is itself the direct product of the linear operators of the components
\begin{align}
    \hat{V}^{(yz|x)} & \equiv \sum_{r,s,r',s'} \mathbf{e}_{r's'}^\intercal V^{(yz|x)}_{rs \rightarrow r's'} \mathbf{e}_{rs}
    \\ &  = \sum_{r,s,r',s'} \mathbf{e}_{r'}^\intercal \otimes \mathbf{e}_{s'}^\intercal U^{(z|xy)}_{s \rightarrow s'}T^{(y|x)}_{r \rightarrow r'} \mathbf{e}_{r} \otimes \mathbf{e}_{s}
    \\ &  = \sum_{r,s,r',s'}  T^{(y|x)}_{r \rightarrow r'} \mathbf{e}_{r'}^\intercal \mathbf{e}_{r} \otimes  U^{(z|xy)}_{s \rightarrow s'}\mathbf{e}_{s'} \mathbf{e}_{s}
    \\ &  = \sum_{r,r'}  T^{(y|x)}_{r \rightarrow r'} \mathbf{e}_{r'}^\intercal \mathbf{e}_{r} \otimes  \sum_{s,s'} U^{(z|xy)}_{s \rightarrow s'}\mathbf{e}_{s'} \mathbf{e}_{s}
    \\ & = \hat{T}^{(y|x)} \otimes \hat{U}^{(z|xy)}.
\end{align}
\end{proof}

\section{Other types of transducer composition}
\label{app:Other Composition Types}

Here, we formally define common types of weighted transducer composition used in the literature, and discuss how they relate to our definition in \autoref{def:Composition}.

There are two main types of composition of weighted finite state transducers (WFSTs): in parallel and in series \citep{mohri-1997-finite,mohri2002weighted}.

\begin{definition}[Parallel composition]
	Let $T = (\mathcal{X}, \mathcal{Y}, \mathcal{R}, T_{r\to r'}^{(y|x)})$ and $U = (\mathcal{Z}, \mathcal{W}, \mathcal{S}, U_{s\to s'}^{(w|z)})$ be transducers. The parallel composition of $T$ and $U$ is a new transducer $V = (\mathcal{X}\times \mathcal{Z}, \mathcal{Y}\times \mathcal{W}, \mathcal{R} \times \mathcal{S}, V_{rs\to r's'}^{(yw|rz)})$ with input alphabet $\mathcal{X}\times \mathcal{Z}$, output alphabet $\mathcal{Y}\times \mathcal{W}$, and stochastic kernel $V_{rs\to r's'}^{(yw|rz)} = T_{r\to r'}^{(y|x)}U_{s\to s'}^{(w|z)}$.
\end{definition}

This is a special case of divergent composition c) from \autoref{sec:special-cases}, where $U$ is not dependent on the output $\mathcal{Y}$ of $T$. Note that, in our formulation in \autoref{def:Composition}, both $T$ and $U$ take in the same input, while the parallel composition defined above assumes separate inputs. This can be easily reconciled by taking the single input to both transducers to be $(\mathcal{X}, \mathcal{Z})$ and letting $T$ ignore the second argument and $U$ ignore the first argument.

\begin{definition}[Serial composition]
	Let $T = (\mathcal{X}, \mathcal{Y}, \mathcal{R}, T_{r\to r'}^{(y|x)})$ and $U = (Y, Z, S, U_{s\to s'}^{(z|y)})$ be transducers. The serial composition of $T$ and $U$ is a new transducer $V = (\mathcal{X}, \mathcal{Z}, \mathcal{R} \times \mathcal{S}, V_{rs\to r's'}^{(z|x)})$ with input alphabet $X$, output alphabet $Z$, and stochastic kernel $V_{rs\to r's'}^{(z|x)} = \sum_{y\in Y}T_{r\to r'}^{(y|x)}U_{s\to s'}^{(z|y)}$.
\end{definition}

This is closely related to the case of series composition a) discussed in \autoref{sec:special-cases}, where $U$ takes only the output of $T$ as its input and ignores the original input $\mathcal{X}$ of $T$. Serial composition is distinct from series composition in that the combined kernel of the serial composition marginalizes out all the possible intermediate symbols from $\mathcal{Y}$, while series composition preserves these observables.

A third type of composition of transducers is the cascade composition. In the literature, this type was originally defined for unweighted transducers or finite semigroups \citep{krohn-rhodes-1965-algebraic, eilenberg1976automata, egri-nagy2005algebraic, maler2010krohnrhodes}, and has also been extended to stochastic semigroups and automata \citep{carlsson2015prime, maler1993decomposition}. Here, we give a definition of explicit stochastic transducers.

\begin{definition}[Cascade composition]
    Let $T = (\mathcal{X}, \mathcal{Y}, \mathcal{R}, T_{r\to r'}^{(y|x)})$ with $\mathcal{Y}=\mathcal{R}$, such that $T_{r\to r'}^{(y|x)} \neq 0$ only if $y = r$, and let $U = (\mathcal{X} \times \mathcal{R}, \mathcal{Z}, \mathcal{S}, U_{s\to s'}^{(z|xy)})$.  Then the cascade composition of $T$ and $U$ is a new transducer $V = (\mathcal{X}, \mathcal{Z}, \mathcal{R}\times \mathcal{S}, V_{rs\to r's'}^{(z|x)})$ with $V_{rs\to r's'}^{(z|x)} = \sum_{y\in Y} T_{r\to r'}^{(y|x)}U_{s\to s'}^{(z|xy)} = T_{r\to r'}^{(r|x)}U_{s\to s'}^{(z|xr)}$.
\end{definition}

This is also a special case of our construction, again marginalizing out the intermittent symbols. Because of the special constraint on the kernel that the output of $T$ has to be its state $r\in\mathcal{R}$ before transitioning, there is only one possible intermittent symbol for each transition, eliminating the need to sum over multiple symbols.

\section{Proof of \autoref{res:when_trans}}
\label{app:when_trans}

\begin{proof}

\begin{figure*}
\centering
\includegraphics[width=.6\columnwidth]{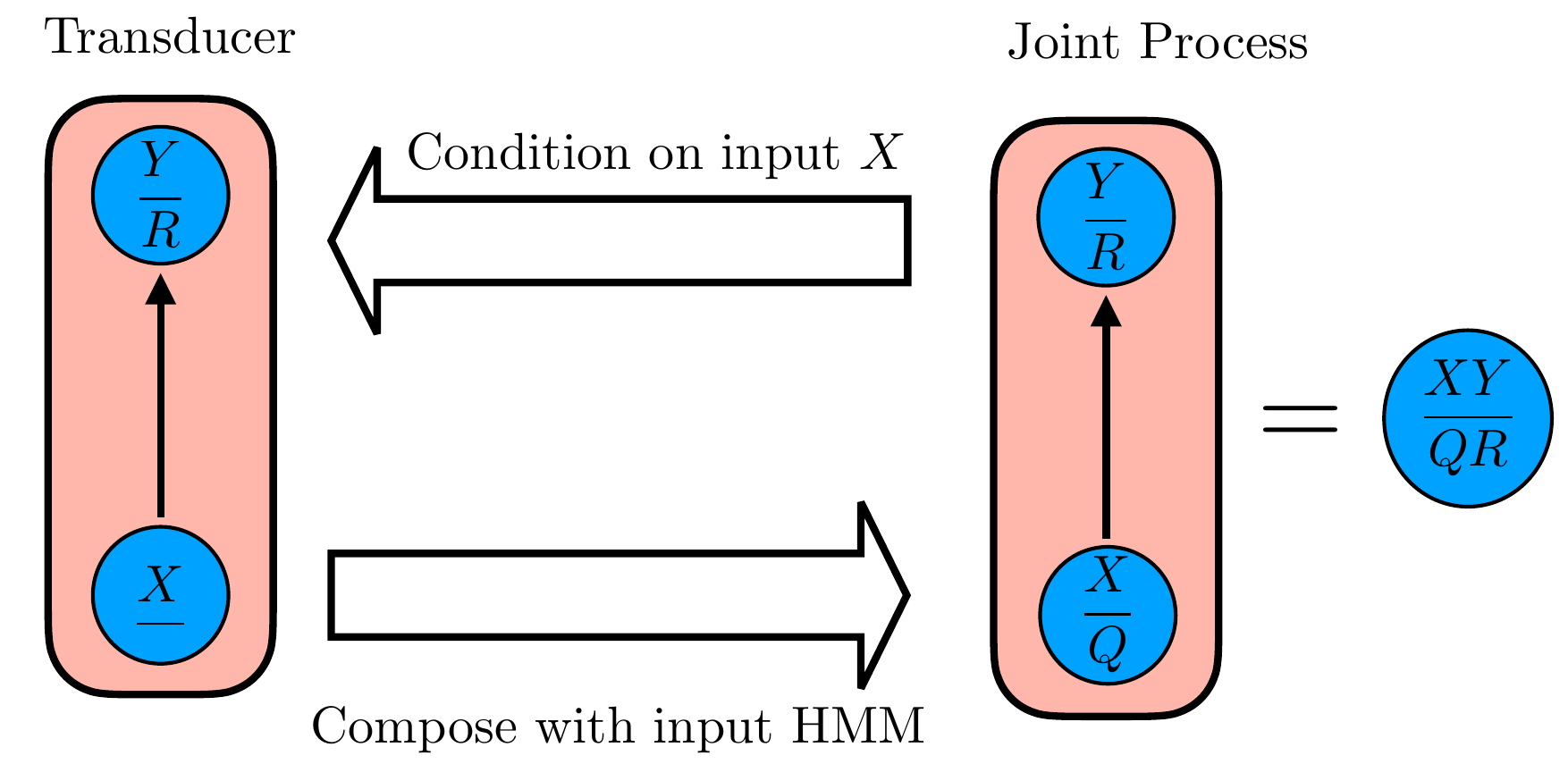}
\caption{
\centering
\small{We can shift between two equivalent viewpoints: one which is conditional on inputs $\frac{X}{\null}$, and one which addresses the joint process.}}
\label{fig:Transducer_Vs_Joint} 
\end{figure*}

Consider the random variables for an input process $X$, latent memory process $R$, and output process $Y$.  Given the distribution $\Pr(X,Y,R)$ over these variables, we now wish to know whether this represents the operation of a transducer.  As shown in \autoref{fig:Transducer_Vs_Joint}, while we started talking about conditional processes that depend on unspecified input sequences, there is an equivalence with the joint distribution when driven by a known input distribution. \citet{rosasai} established that this distribution represents the operation of a transducer iff the following equality holds for all times $t$: 
\begin{align}\label{eq:transducer_condition}
I[\overrightarrow{R}_{t+1},\overrightarrow{Y}_t;\overleftarrow{Y}_t,\overleftarrow{R}_t,\overleftarrow{X}_t|\overrightarrow{X}_t,R_t]=0.
\end{align}
As shown in \autoref{fig:Condensed_Transducer}, the present latent state $R_t$ of a transducer shields the future of the latent states $\overrightarrow{R}_{t+1}$, and outputs $\overrightarrow{Y}_t$, from the past $\overleftarrow{Y}_t$,$\overleftarrow{X}_t$,$\overleftarrow{R}_t$, when you condition on future inputs $\overrightarrow{X}_t$.  Conditioning on the input future is necessary, because information can be carried between the input past and input future outside of the latent memory of the transducer.  This measure allows us to determine the causal influence of the processes on each other.  

Informally, Intransducibility is large when the joint process $(X,Y,R)$ behaves in a way that cannot be implemented by any causal transducer with latent state $R$:

Since each term is non-negative, the Intransducibility is only zero when all terms are zero.  Thus, the condition for the distribution $\Pr(X,Y,R)$ to be consistent with a transducer that transforms $X$ to $Y$ with latent state $R$ is met.  In the reverse direction, if we assume the distribution is generated by a transducer, each term of the measure is zero, so the measure is zero overall.
\end{proof}

\section{Proving composite causal states}
\label{app:Proving Composite Causal States}

We wish to show that the composite map $\epsilon_{UT}$ is the $\epsilon$-map for the interface $\mathcal{I}[YZ|X]$.  For this to be true, it must be true that the following conditions are equivalent:
\begin{align}
\epsilon_{UT}(\overleftarrow{y}_t,\overleftarrow{z}_t|\overleftarrow{x}_t) & =\epsilon_{UT}(\overleftarrow{y}'_t,\overleftarrow{z}'_t|\overleftarrow{x}'_t)
\\ & \Leftrightarrow
\\ p(\overrightarrow{y}_t,\overrightarrow{z}_t|\overrightarrow{x}_t,\overleftarrow{y}_t,\overleftarrow{z}_t,\overleftarrow{x}_t)& =p(\overrightarrow{y}_t,\overrightarrow{z}_t|\overrightarrow{x}_t,\overleftarrow{y}'_t,\overleftarrow{z}'_t,\overleftarrow{x}'_t)\text{ for all } \overrightarrow{x},\overrightarrow{y},\overrightarrow{z}.
\end{align}
We can factor the probabilities in the equivalence expression
\begin{align}
p(\overrightarrow{z}_t,\overrightarrow{y}_t|\overrightarrow{x}_t,\overleftarrow{z}_t,\overleftarrow{y}_t,\overleftarrow{x}_t)=
p(\overrightarrow{z}_t|\overrightarrow{y}_t,\overrightarrow{x}_t,\overleftarrow{z}_t,\overleftarrow{y}_t,\overleftarrow{x}_t)
p(\overrightarrow{y}_t|\overrightarrow{x}_t,\overleftarrow{z}_t,\overleftarrow{y}_t,\overleftarrow{x}_t),
\end{align}
where we use the shorthand notation $p(a|b)=\Pr(A=a|B=b)$.  Note that the future of $\overrightarrow{y}_t$ doesn't depend on the past $\overleftarrow{z}_t$ outside of its own past $\overleftarrow{y}_t$, meaning that
\begin{align}
p(\overrightarrow{y}_t|\overrightarrow{x}_t,\overleftarrow{z}_t,\overleftarrow{y}_t,\overleftarrow{x}_t)=p(\overrightarrow{y}_t|\overrightarrow{x}_t,\overleftarrow{y}_t,\overleftarrow{x}_t),
\end{align}
and
\begin{align}
p(\overrightarrow{z}_t,\overrightarrow{y}_t|\overrightarrow{x}_t,\overleftarrow{z}_t,\overleftarrow{y}_t,\overleftarrow{x}_t)=
p(\overrightarrow{z}_t|\overrightarrow{y}_t,\overrightarrow{x}_t,\overleftarrow{z}_t,\overleftarrow{y}_t,\overleftarrow{x}_t)
p(\overrightarrow{y}_t|\overrightarrow{x}_t,\overleftarrow{y}_t,\overleftarrow{x}_t).
\end{align}
This means that the condition for $\epsilon_{UT}$ to be the $\epsilon$-map reduces to
\begin{align}
\epsilon_{UT}(\overleftarrow{y}_t,\overleftarrow{z}_t|\overleftarrow{x}_t) & =\epsilon_{UT}(\overleftarrow{y}'_t,\overleftarrow{z}'_t|\overleftarrow{x}'_t)
\\ & \Leftrightarrow
\\ p(\overrightarrow{z}_t|\overrightarrow{y}_t,\overrightarrow{x}_t,\overleftarrow{z}_t,\overleftarrow{y}_t,\overleftarrow{x}_t)
p(\overrightarrow{y}_t|\overrightarrow{x}_t,\overleftarrow{y}_t,\overleftarrow{x}_t)& =p(\overrightarrow{z}_t|\overrightarrow{y}_t,\overrightarrow{x}_t,\overleftarrow{z}'_t,\overleftarrow{y}'_t,\overleftarrow{x}'_t)
p(\overrightarrow{y}_t|\overrightarrow{x}_t,\overleftarrow{y}'_t,\overleftarrow{x}'_t).
\end{align}
We break the proof that the composite map is the correct forward and reverse step
\begin{enumerate}
\item $\Rightarrow$: If $\epsilon_{UT}$ maps two histories $\overleftarrow{xyz}_t$ and $\overleftarrow{xyz}'_t$ to identical causal states, then $\epsilon_T(\overleftarrow{y}|\overleftarrow{x})=\epsilon_T(\overleftarrow{y}'|\overleftarrow{x}')$ and $\epsilon_U(\overleftarrow{z}|\overleftarrow{y},\overleftarrow{x})=\epsilon_U(\overleftarrow{z}'|\overleftarrow{y}',\overleftarrow{x}')$.  Furthermore, this implies that
\begin{align}
p(\overrightarrow{y}_t|\overrightarrow{x}_t,\overleftarrow{y}_t,\overleftarrow{x}_t) & =p(\overrightarrow{y}_t|\overrightarrow{x}_t,\overleftarrow{y}'_t,\overleftarrow{x}'_t) 
\\ p(\overrightarrow{z}_t|\overrightarrow{y}_t,\overrightarrow{x}_t,\overleftarrow{z}_t,\overleftarrow{y}_t,\overleftarrow{x}_t)& =p(\overrightarrow{z}_t|\overrightarrow{y}_t,\overrightarrow{x}_t,\overleftarrow{z}'_t,\overleftarrow{y}'_t,\overleftarrow{x}'_t),
\end{align}
and therefore by multiplying these equalities together, we arrive at the condition that the future distributions are equivalent
\begin{align}
p(\overrightarrow{z}_t|\overrightarrow{y}_t,\overrightarrow{x}_t,\overleftarrow{z}_t,\overleftarrow{y}_t,\overleftarrow{x}_t)
p(\overrightarrow{y}_t|\overrightarrow{x}_t,\overleftarrow{y}_t,\overleftarrow{x}_t)& =p(\overrightarrow{z}_t|\overrightarrow{y}_t,\overrightarrow{x}_t,\overleftarrow{z}'_t,\overleftarrow{y}'_t,\overleftarrow{x}'_t)
p(\overrightarrow{y}_t|\overrightarrow{x}_t,\overleftarrow{y}'_t,\overleftarrow{x}'_t),
\end{align}
and equivalence is half-proved.
\item $\Leftarrow$:  We begin with the equality of future distributions and
sum over $z$ future to obtain:
\begin{align}
p(\overrightarrow{y}_t|\overrightarrow{x}_t,\overleftarrow{y}_t,\overleftarrow{x}_t)=p(\overrightarrow{y}_t|\overrightarrow{x}_t,\overleftarrow{y}'_t,\overleftarrow{x}'_t).
\end{align}
This implies that $\epsilon_T(\overleftarrow{y}'_t|\overleftarrow{x}'_t)=\epsilon_T(\overleftarrow{y}_t|\overleftarrow{x}_t)$, leveraging the assumption that $T$ is the $\epsilon$-transducer for $\mathcal{I}[Y|X]$. 

Plugging the equality of $y$ future distribution back into the original equality and factoring out, an equality for the future distribution of $z$:
\begin{align}
p(\overrightarrow{z}_t|\overrightarrow{y}_t,\overrightarrow{x}_t,\overleftarrow{z}_t,\overleftarrow{y}_t,\overleftarrow{x}_t)& =p(\overrightarrow{z}_t|\overrightarrow{y}_t,\overrightarrow{x}_t,\overleftarrow{z}'_t,\overleftarrow{y}'_t,\overleftarrow{x}'_t),
\end{align}
meaning that $\epsilon_U(\overleftarrow{z}|\overleftarrow{y},\overleftarrow{x})=\epsilon_U(\overleftarrow{z}'|\overleftarrow{y}',\overleftarrow{x}')$, and therefore the causal states are identical $\epsilon_{UT}(\overleftarrow{z}'_t,\overleftarrow{y}'_t|\overleftarrow{x}'_t)=\epsilon_{UT}(\overleftarrow{z}_t,\overleftarrow{y}_t|\overleftarrow{x}_t)$. 
\end{enumerate}

\end{document}